\documentclass[11pt, a4paper, logo, copyright]{pluralisresearchiclr}

\usepackage{amsmath,amsfonts,bm}

\def\eqref#1{equation~\ref{#1}}

\def\1{\bm{1}}

\DeclareMathAlphabet{\mathsfit}{\encodingdefault}{\sfdefault}{m}{sl}
\SetMathAlphabet{\mathsfit}{bold}{\encodingdefault}{\sfdefault}{bx}{n}

\newcommand{\R}{\mathbb{R}}

\pdftrailerid{redacted}

\makeatletter
\renewcommand\bibentry[1]{\nocite{#1}{\frenchspacing\@nameuse{BR@r@#1\@extra@b@citeb}}}
\makeatother

\usepackage{kantlipsum, lipsum}
\usepackage{dsfont}
\usepackage{algorithm}
\usepackage{algpseudocode}
\usepackage{adjustbox}
\usepackage{enumitem}
\usepackage{multirow}
\usepackage{nicefrac}
\usepackage{float}
\usepackage[utf8]{inputenc} 
\usepackage[T1]{fontenc}    
\usepackage{hyperref}       
\usepackage{url}            
\usepackage{booktabs}       
\usepackage{amsfonts}       
\usepackage{subcaption}     
\usepackage{amssymb}        
\usepackage{amsthm}         
\usepackage{nicefrac}       
\usepackage{microtype}      
\usepackage{xcolor}         
\usepackage{amsmath}

\usepackage{wrapfig}
\usepackage{graphicx}

\newtheorem{theorem}{Theorem}
\newtheorem{proposition}{Proposition}
\newtheorem{lemma}{Lemma}
\theoremstyle{remark}

\newtheorem{corollary}{Corollary}

\usepackage[most]{tcolorbox}
\tcbuselibrary{skins, breakable}

\newtcolorbox{informaltheorembox}{
  enhanced,
  breakable,
  colback=blue!2,
  colframe=blue!50!black,
  sharp corners=south,
  left=6pt, right=6pt, top=6pt, bottom=6pt
}

\newcolumntype{R}[2]{%
    >{\adjustbox{angle=#1,lap=\width-(#2)}\bgroup}%
    l%
    <{\egroup}%
}

\usepackage{pifont}

\usepackage{subcaption}
\usepackage{url}

\graphicspath{{figures/}}

\title{Mixtures of Subspaces for Bandwidth Efficient Context Parallel Training}

\correspondingauthor{author@email.com}


\reportnumber{} 

\renewcommand{\today}{} 

\author[1]{Sameera Ramasinghe}
\author[1]{Ajanthan Thalaiyasingam}
\author[1]{Hadi Mohaghegh Dolatabadi}
\author[1]{Gil Avraham}
\author[1]{Violetta Shevchenko}
\author[1]{Yan Zuo}
\author[1]{Chamin Hewa Koneputugodage}
\author[1]{Alexander Long}

\affil[1]{Pluralis Research}

\begin{abstract}
    Pretraining language models with extended context windows enhances their ability to leverage rich information during generation. Existing methods split input sequences into chunks, broadcast them across multiple devices, and compute attention block by block which incurs significant communication overhead. While feasible in high-speed clusters, these methods are impractical for decentralized training over low-bandwidth connections. We propose a compression method for communication-efficient context parallelism in decentralized settings, achieving a remarkable compression rate of over $95\%$ with negligible overhead and no loss in convergence. Our key insight is to exploit the intrinsic low-rank structure of activation outputs by dynamically constraining them to learned mixtures of subspaces via efficient reparameterizations. We demonstrate scaling billion-parameter decentralized models to context lengths exceeding $100\mathrm{K}$ tokens on networks as slow as $300\mathrm{Mbps}$, matching the wall-clock convergence speed of centralized models on $100\mathrm{Gbps}$ interconnects.
\end{abstract}

\begin{document}
\maketitle

\section{Introduction}
\label{sec:introduction}
Rapid scaling of large language models (LLMs) has made distributed training a necessity \cite{krizhevsky2012imagenet, kolesnikov2020big, dubey2024llama, ren2023pangu}. As both model size and context length continue to grow, efficient training increasingly depends on parallelization across multiple devices. Traditional distributed training paradigms assume high-bandwidth, low-latency interconnects, typically available in centralized data centers. In contrast to such centralized settings, the emerging paradigm of \emph{decentralized training} \cite{yuan2022decentralized, ryabinin2023swarm, lian2017can, koloskova2019decentralized, koloskova2020unified} enables collaborative and democratized machine learning by distributing computation across heterogeneous, geographically dispersed nodes over the Internet, without requiring specialized networking hardware or centralized orchestration.


However, decentralized training presents a core technical challenge: \emph{limited communication bandwidth}. When nodes are connected via commodity networks, communication quickly becomes a bottleneck. Most prior work, has addressed this issue in the context of distributed data parallelism (DDP), where each node maintains a full model replica and synchronizes gradients during training. A variety of bandwidth-efficient techniques, such as gradient quantization \cite{wu2018error, nassif2023quantization, tang2018decentralization}, sparsification \cite{wang2021error, shi2019distributed, singh2024efficient}, and delayed synchronization \cite{ryabinin2021moshpit, douillard2023diloco, liu2024asynchronous, douillard2025streaming}, have been proposed to reduce overhead in this setting. Pipeline parallelism (PP)~\cite{huang2019gpipe}, where model layers are partitioned across devices, has also been explored to a limited extent \cite{ramasinghe2025beyond, wang2021pufferfish}.

A significantly more challenging -- and, to our knowledge, entirely unexplored -- setting is \emph{context parallelism} (CP) in decentralized environments. CP has become critical for pretraining frontier LLMs, as it enables efficient training with extremely long sequences, often exceeding 100K tokens (e.g., \textsc{LlaMA} 3 \cite{dubey2024llama}: $130\mathrm{K}$, \textsc{LlaMA} 4: $256\mathrm{K}$, \textsc{DeepSeek} \cite{liu2024deepseek}: $128\mathrm{K}$, \textsc{Qwen} 3 \cite{yang2024qwen2}: $128\mathrm{K}$), thereby enhancing the model’s ability to capture long-range dependencies. In context-parallel training, each node processes a local chunk of the input and broadcasts its attention activations to all other nodes in every layer and at every step. This imposes substantial communication demands, as attention mechanisms require \emph{global access} to all key and value activations. While centralized systems handle this using high-bandwidth interconnects, all-to-all communication becomes prohibitively expensive in decentralized environments with low-bandwidth links.

In this work, we propose a method to \emph{drastically reduce} the communication required for context-parallel attention without sacrificing model quality, enabling decentralized systems connected via standard internet-grade links to match the convergence performance of centralized systems with datacenter-grade bandwidth. Our approach leverages the observation that attention activations (queries, keys, and values) often reside on low-dimensional manifolds. We exploit this structure by factorizing the attention weights so that outputs lie within dynamic mixtures of low-dimensional subspaces. To ensure convergence, we optimize the factored weights on a Riemannian product manifold and introduce an efficient reparameterization scheme that significantly reduces computational and communication overhead. Additionally, we provide theoretical guarantees on the expressivity and convergence of our method, offering principled justification for each design choice.




Crucially, this approach introduces only minor architectural changes with negligible training overhead, and these components can be \emph{removed after training}, yielding a standard transformer architecture compatible with existing inference infrastructure and downstream deployment frameworks. We employ our method up to billion-parameter scale models under various settings and demonstrate that our method achieves over \textbf{$\mathbf{95\%}$ communication compression} without harming performance, enabling training with long context windows across devices connected via commodity internet ($300\mathrm{Mbps}$), while matching the wall-clock convergence of centralized systems with high-speed interconnects ($100\mathrm{Gbps}$).


\section{Background and motivation}

\subsection{Context-parallel training and the communication bottleneck}
\label{sec:context-parallel}









We begin with a brief exposition on CP training and refer the reader to~\cite{dubey2024llama} for an extended read. Transformer attention requires each query to interact with all key-value pairs, resulting in a computational complexity that grows quadratically with context window. This becomes particularly prohibitive for long sequences, which makes parallelization strategies essential. In \emph{context-parallel} settings, the input sequence \(X \in \mathbb{R}^{n \times d}\), where $n$ is the context length and $d$ is the model dimension, is partitioned across \(m\) devices along the context dimension:

\[
X \;=\;
\bigl[X_{1}^{\top}\;\cdots\;X_{m}^{\top}\bigr]^{\top},
\qquad
X_{i}\in\mathbb{R}^{n_i\times d},
\quad
\sum_{i=1}^{m} n_i = n.
\]

Each device \(i\) computes local queries, keys, and values per head. For clarity we suppress the head index; all quantities are understood to be per attention head unless otherwise stated:
\[
Q_{i}=X_{i}W_q\in\mathbb{R}^{n_i\times d},\quad
K_{i}=X_{i}W_k\in\mathbb{R}^{n_i\times d},\quad
V_{i}=X_{i}W_v\in\mathbb{R}^{n_i\times d}.
\]
Computing attention locally requires global access to keys and values:
\[
K_{\text{g}}=\bigl[K_{1}^{\top}\;\cdots\;K_{m}^{\top}\bigr]^{\top}\!\in\mathbb{R}^{n\times d},
\qquad
V_{\text{g}}=\bigl[V_{1}^{\top}\;\cdots\;V_{m}^{\top}\bigr]^{\top}\!\in\mathbb{R}^{n\times d},
\]

Typically, CP performs (some form of) \emph{all-gather}, where each device broadcasts its local \(K_{i}, V_{i}\) to form \(K_{\mathrm{g}}, V_{\mathrm{g}}\), incurring communication cost \(O(nd)\) per device where $d \ll n$. Recently proposed \emph{Ring Attention}~\citep{liu2024ringattention} pipelines this communication in a ring topology, incrementally exchanging local key-value blocks and computing partial attentions at each stage. Above methods fundamentally rely on the costly communication of large \(K, V\) matrices.

\subsection{Attention Outputs Exhibit Low-Rank Structure}
\label{sec:qkv_lowrank}
Our compression scheme is inspired by  the observation that the attention outputs of pretrained transformers lie on a low-dimensional manifold. To support this, we analyze publicly available checkpoints of large-scale pretrained LLMs and examine their attention activations. Fig.~\ref{fig:qkv_lowrank} presents an illustration of \textsc{LLaMA} $70\mathrm{B}$.
\begin{figure}
  \centering
  \includegraphics[width=0.5\textwidth]{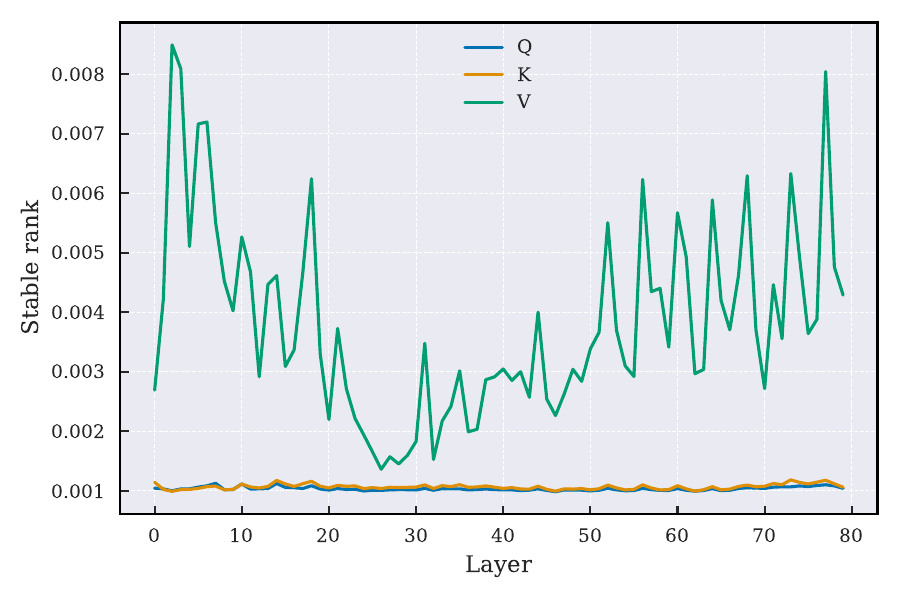}
  \caption{\textbf{Attention outputs of LlaMa-70B.}
Shown is the empirical rank of the $Q$, $K$, and $V$ activations, normalised by their maximum possible rank, for every layer of the official \textsc{LlaMa} $70\mathrm{B}$ checkpoint. All three projections are extremely low-rank: $Q$ and $K$ sit at roughly $0.1 \%$ of full rank, while $V$ is slightly larger at $\sim 0.5 \%$.}
  \label{fig:qkv_lowrank}
\end{figure}

Specifically, we measure the \emph{stable rank} of the query (\(Q\)), key (\(K\)), and value (\(V\)) activations across each attention layer. The stable rank of a matrix \( A \in \mathbb{R}^{n \times d} \) is defined as:
$
\mathrm{srank}(A) = \frac{\|A\|_F^2}{\|A\|_2^2},
$
where \( \|A\|_F \) denotes the Frobenius norm and \( \|A\|_2 \) the spectral norm. Unlike the conventional matrix rank -- which is highly sensitive to small perturbations and numerical noise -- the stable rank offers a robust, continuous measure of effective dimensionality. This makes it particularly suitable for characterizing learned neural representations, where numerous singular values are typically small yet non-zero due to noise or over-parameterization.

As depicted in Fig.~\ref{fig:qkv_lowrank}, the stable ranks of attention activations remain low across all layers. Interestingly, \( Q \) and \( K \) generally exhibit slightly lower ranks than \( V \), indicating a higher degree of compressibility \footnote{We only need to compress $K$ and $V$ since $Q$ can remain local.}. This observation underpins our approach, leveraging low-rank factorization for efficient compression. Further evidence of this phenomenon in other architectures is provided in Appendix~\ref{app:rank}. Next, we formalize this idea.

\section{Method}
\label{sec:method}

We now present our proposed method for efficient context-parallel transformer training. 
First, we formalize how the empirically observed low-rank structure in attention activations enables effective compression. Next, we explain why using a fixed subspace for compression can be overly restrictive, motivating our joint learning strategy that adaptively optimizes both the projection subspace and the attention weights (\S\ref{sec:joint_optimization}). We then introduce a computationally efficient reparameterization approach that maintains optimality guarantees while significantly reducing overhead (\S\ref{sec:reparam-geometry}). Finally, we describe how to reduce communication costs by dynamically compressing attention activations through per-chunk rotations and demonstrate how the model can seamlessly revert to a \emph{standard transformer architecture} at inference time (\S\ref{sec:comm-cost}–\ref{sec:unplug-projections}).


In \S.~\ref{sec:qkv_lowrank}, we saw that the \(Q\), \(K\), \(V\) activations of large pretrained transformers exhibit a pronounced low-rank structure (Fig.~\ref{fig:qkv_lowrank}). This finding implies that it is feasible to transmit only the low-dimensional components of these activations between devices, thereby achieving near-lossless compression in practice. Formally, let the columns of an orthonormal matrix \(U \in \mathbb{R}^{d \times r}\), with \(r \ll d\), span the dominant subspace of the activations. Rather than communicating the full local activation matrix \(Z = X^{(i)} W \in \mathbb{R}^{n_i \times d}\), where $Z \in \{K,V\}$ denotes key/value activations, $W \in \{W_k, W_v \}$ and $X$ is the input to the attention layer, we can transmit only its compressed representation:
$
Z_{\text{comp}} = X^{(i)} W U \;\in\; \mathbb{R}^{n_i \times r}.
$
The original activations can then be reconstructed at the receiving node as:
$
Z \approx Z_{\text{comp}} U^{\top}.
$
This compression method preserves all information within the subspace spanned by \(U\), and is  lossless when activations lie entirely in this subspace. Equivalently, this projection can be folded onto the attention weights and be interpreted as factorizing them into a low-rank representation:
$
W = B (U U^{\top}), \quad B \in \mathbb{R}^{d \times d}.
$


\paragraph{Sub‑optimality of a fixed subspace.} The formulation above implicitly assumes that an \emph{a priori} choice of \(U\) is sufficiently expressive for every layer and every chunk in every optimization stage.  
It is straightforward to see where this assumption can break down. Even if there is an optimal low-rank attention weight matrix, restricting weights to the form \( W = BUU^\top \) limits the search to the column space of \( U \). If this space does not contain the true optimum, the model may converge to a suboptimal solution. In short, fixing \( U \) can prevent the model from reaching the best possible performance.



\subsection{Joint Optimization over a Product Manifold}
\label{sec:joint_optimization}

To address the limitations of a fixed \(U\), we propose jointly optimizing factorization \( W = B\,U\,U^\top \). Specifically, we simultaneously learn both the subspace representation \(U\) and the matrix \(B\) on the product manifold:
$
\mathcal{M} \;=\; \mathbb{R}^{d\times d} \,\times\, \mathrm{St}(n,r).
$
Here, \(B \in \mathbb{R}^{d\times d}\) is optimized in standard Euclidean space, whereas \(U\) resides on the Stiefel manifold \(\mathrm{St}(n,r)\), where updates can be naturally performed via Riemannian gradient descent \footnote{The Stiefel manifold \(\mathrm{St}(d,r)\) is defined as the set of all \( d\times r \) matrices with orthonormal columns, formally given by \(\mathrm{St}(d,r) = \{U \in \mathbb{R}^{d\times r} : U^\top U = I_r\}\).}. The following result establishes that this joint optimization achieves linear convergence under gradient descent.



\begin{informaltheorembox}

\textbf{Convergence.} Let \(\Phi(W,\vartheta)\) be a smooth loss and consider the factorization $W = B\,U\,U^{\top}$ for attention weights where
   $B \in \mathbb{R}^{d\times d},\;
   U \in \mathrm{St}(d,r),\;$
   and $\vartheta \in \mathbb{R}^{p}$ denotes all other parameters.
Minimizing the reparameterized objective
\(\hat{\Phi}(B,U,\vartheta) = \Phi\bigl(BUU^{\top},\vartheta\bigr)\) over the product manifold
$
   \mathcal{M}
   := \mathbb{R}^{d\times d}\times\mathrm{St}(d,r)\times\mathbb{R}^{p}
$
with Riemannian gradient descent,
and under mild assumptions,
yields \emph{Q-linear} (geometric) convergence to a first-order
stationary point. For the formal result and proof, see Lemma \ref{thm:ProdManifoldConvergenceFull} (Appendix).
\end{informaltheorembox}

Note that since \( \| U \| = 1 \), the factorized objective remains Lipschitz smooth, and the convergence result naturally follows from the standard gradient descent theory on both Euclidean and Riemannian manifolds. We include a full proof in Appendix~\ref{app:theory} for completeness, explicitly treating the product manifold structure and assuming a Polyak--Łojasiewicz (PL) condition.



\subsection{Reducing Computational Cost via Reparameterization of \(U\)}
\label{sec:reparam-geometry}

Direct optimisation of \(U\) on the Stiefel manifold \(\mathrm{St}(n,r)\) via Riemannian gradient descent provides strong theoretical guarantees but is costly: after every Euclidean update, \(U\) must be re-orthonormalised (the standard “retraction’’ on to the manifold), which is performed with a QR or SVD factorisation to restore \(U^{\top}U = I_r\). To mitigate this, we use an efficient reparameterization of \(U\) using a fixed orthonormal basis \(\overline{U}\) and a learnable rotation \(R(\theta)\in O(d)\):
\[
U(\theta) \;=\; R(\theta)\,\overline{U},
\]
where \(O(d)\) denotes the orthogonal manifold consisting of all \(d\times d\) orthonormal matrices. If the mapping \(\theta \mapsto R(\theta)\) is sufficiently expressive, rotations \(R(\theta)\) can fully parameterize \(O(d)\), preserving the representational power of the manifold while significantly reducing computational overhead.

\paragraph{Preservation of geometry and stationary points.}
Reparameterizing \(U\) as \(U(\theta)=R(\theta)\,\overline U\) moves the orthonormal constraint onto an unconstrained Euclidean variable \(\theta\), eliminating expensive QR/SVD steps and letting us run ordinary SGD/Adam in \(\theta\)-space. %
\textit{A natural concern is that this change of variables might distort the loss landscape and hinder optimization;} however, we show that this is not the case. %
The chain rule shows $\nabla_\theta\widehat\Phi(B,\theta,\vartheta)
      =D_\theta U(\theta)^{\!\top}\,
        \mathrm{grad}_U\Phi\!\bigl(B,U(\theta),\vartheta\bigr),$
so \(\nabla_\theta\widehat\Phi\) is exactly the pull-back of the original Riemannian gradient. Thus, the first-order critical points remain unchanged. The following statement formalizes this result.

\begin{informaltheorembox}

\textbf{Equivalence of stationery points.} Under the reparameterization $U \;=\; R(\theta)\,\overline{U}$, minimizing $\hat{\Phi}$ possesses exactly the same local minima and strict
saddle points as minimizing $\Phi$. For the formal result and proof, see Theorem \ref{thm:reparam-U} (Appendix).
\end{informaltheorembox}

Thus, we can effectively represent and optimize the projection subspace implicitly through rotations without compromising the quality or optimality of solutions.

\subsection{Reducing the Communication Cost}
\label{sec:comm-cost}

The reparameterization \( U(\theta) = R(\theta)\,\overline{U} \) allows us to locally cache the fixed orthonormal frame \(\overline{U}\) at each node, and transmit only the parameters \(\theta\). However, to fully parameterize rotations in the orthogonal group \(O(d)\), one typically requires \(\frac{1}{2}d(d-1)\) parameters, \textit{i.e.,} $\theta \in \mathbb{R}^{d(d-1)/2}$. We show next that in practice performing a dense search over all possible rotations is unnecessary. Specifically, we can obtain a  trade-off between the search space and the communication efficiency by controlling the dimensionality of $\theta$.

To achieve a more compact representation for the communication cost reduction, we select a small set of fixed skew-symmetric matrices \(\{A_1,\dots,A_k\}\subset\mathfrak{o}(d), A_i^T = -A_i\) (where $\mathfrak{o}(d)$ denotes the Lie algebra of the orthogonal group) and define the corresponding \(k\)-dimensional Lie subgroup \cite{Hall2015, EdelmanAriasSmith1998}:
\[
\mathcal{R}_k = \left\{\, R(\theta)=\exp\left(\textstyle\sum_{l=1}^k \theta(l) A_i\right) \mid \theta\in\mathbb{R}^{k}\,\right\},
\]

where $\theta(l)$ is the $l^{th}$ element of $\theta$. Because the exponential map is a local diffeomorphism around \(\theta=0\), the set \(U(\theta)=R(\theta)\,\overline{U}\) forms a \(k\)-dimensional submanifold of \(\mathrm{St}(d,r)\) for sufficiently small \(\|\theta\|\). Choosing \(k \ll \frac{1}{2}d(d-1)\) thus provides a favorable trade-off between communication cost and representational flexibility. Importantly, our earlier result on the absence of spurious minima remains valid provided an optimal frame \(U_{\star}\) lies within (or sufficiently close to) the reachable manifold \(\{R\overline{U} : R\in\mathcal{R}_k\}\), as the mapping \(\theta \mapsto U(\theta)\) remains locally surjective onto this manifold.

\subsection{Dynamic mixtures of subspaces via per-chunk rotations}
\label{sec:dyn-mix}

\S~\ref{sec:comm-cost} depicted that the rotation dimension \( k \) controls a trade-off between representational flexibility and communication efficiency. With a well-chosen a priori \(\overline{U}\), it becomes feasible to use a small \( k \), restricting the optimization to a local neighborhood around \(\overline{U}\). 

We generate this prior through a short, uncompressed \textbf{warm-up phase}, in which the model is trained for a small number of iterations (\(<500\)) using a reduced context length to avoid communication bottlenecks. After this phase, each node computes the top \( r \) principal components of its local attention weights and stores them as a fixed subspace basis \( \overline{U} \in \mathrm{St}(d, r) \). Empirical evidence from prior work on weight--subspace stabilization (e.g.,~\cite{feng2022rank,hu2024sp3}) suggests that dominant activation subspaces emerge early in training, supporting this strategy. 


\paragraph{Per-sample adaptation.}
Using a single global rotation for all inputs may underfit heterogeneous data. To retain expressivity \emph{without increasing \(k\)}, we introduce a lightweight mechanism to predict a unique rotation parameter \(\theta\) for each sequence chunk. Recall that, in context-parallel training, each node $i$ processes a distinct chunk $X_{i} \in \mathbb{R}^{n_i \times d}$ from the input sequence. For an attention output chunk $Z_{i} = WX_{i}$, we employ a small linear prediction head:
$
\psi: \mathbb{R}^{d}\rightarrow\mathbb{R}^{k}, \quad 
\theta = \psi\left( Z_{\text{avg},i}\right),
$ where $Z_{\text{avg},i}$ is the average attention output of the chunk, generating chunk-specific rotation parameters.
Given a set of preshared skew-symmetric generators \(\{A_l\}_{l=1}^{k}\subset\mathfrak{o}(d)\) cached locally on each node, we construct the rotation as:
$
R(\theta_i) = \exp\left(\sum_{l=1}^{k}\theta_i(l) A_l\right) \in \mathcal{R}_k \subset O(d)
$.
Locally, keys and values are compressed as
$
Z_{\text{comp},i} = Z_{i}\,R(\theta_i)\,\overline{U}\;\in\;\mathbb{R}^{n\times r}.
$
$(Z_{(\text{comp},i)}, \theta_i)$ is then broadcasted, and the receiving nodes reconstructs the keys/values as:
$
Z_i \approx Z_{(i, \text{comp})}\,\overline{U}^{\top} R(\theta_i)^{\top}.
$ Note that peak memory is dominated by the attention computation, scaling as $\mathcal{O}(n_i^2)$, making the linear head's overhead $\mathcal{O}(dk)$ negligible—an observation we also demonstrate empirically. The overall procedure is summarized in Algorithm \ref{alg:compressed_attention}.

\paragraph{Bandwidth cost.}  
In our method, each node transmits \(n r\) floats (activations) in \(\tilde Z\) and \(k\) additional scalars in \(\theta\). Typically, we have \(k \ll n r \ll n d\), ensuring low communication overhead. Remarkably, we found that even using \(k=1\) -- a single rotation angle that defines a plane -- is sufficient to preserve training stability and input-adaptive flexibility, achieving bandwidth efficiency comparable to that of a fixed global rotation.

In implementation, we set \(S\sim\mathcal N(0,1)^{d\times d}\) to be fixed and define the
skew-symmetric generator  
$
A\;:=\;\frac{\theta}{\|S-S^{\top}\|_{\mathrm F}}\,(S-S^{\top}),
\qquad
A^{\top}=-A,\quad \theta 
\geq 0\;.
$
For \(\theta\in\mathbb R\) we set the rotation
\(\displaystyle R(\theta)=\exp(\theta A)\in O(d)\),
so \( A\) fixes the rotation plane while \(\theta\) sets its
magnitude.

\paragraph{Second-order approximation.}  
Since \(A\) is skew-symmetric, its spectral norm satisfies \(\|A\|_2 = \theta\). For sufficiently small \(|\theta|\leq\epsilon \ll 1\), the rotation matrix \(R(\theta)\) admits a second-order Taylor approximation:
\[
R(\theta) \approx I + \theta A + \tfrac{1}{2}\theta^{2} A^{2}.
\tag{1}
\]
This approximation provides two key advantages. 1) \textbf{Computational cost: } scaling as \(\mathcal{O}(d^2)\), in contrast to the exact matrix exponential computation (e.g., by Padé or Schur decomposition), which scales as \(\mathcal{O}(d^3)\). 2) \textbf{Near identity bias:} it induces a beneficial near-identity bias, effectively acting as an approximately unbiased estimator of identity \(I\) when \(\theta\) is small and centered around zero (enforced via clipping). In this regime, higher-order terms vanish in expectation, yielding \(\mathbb{E}[R(\theta)] \approx I\). This property allows rotations to remain close to the initial warm-start subspace \(\overline{U}\), facilitating controlled local adaptation without significant drift. By fixing \(A\) and using a scalar \(\theta\), we achieve a communication complexity of \(\mathcal{O}(nr)\), significantly lower than the naive \(\mathcal{O}(nd)\).

Attention weights must still be synchronized across devices, but they
evolve far more slowly than activations \cite{ChenAdaptiveFreeze2023, BaHintonFastWeights2016}.
We therefore average the corresponding weights only every \(c\) steps;
in all experiments we use \(c = 200\), which incurs negligible
communication overhead.

\begin{algorithm}[t]
\label{algo}
\caption{Compression-aware context parallel attention (per node, per head)}
\label{alg:compressed_attention}
\begin{algorithmic}[1]
\Require Input $X \in \mathbb{R}^{n_i \times d}$, Attention weight $W \in \mathbb{R}^{d \times d}$,  Warm-started basis $\bar{U} \in \mathbb{R}^{d \times r}$, learnable linear head $\psi: \mathbb{R}^d \rightarrow \mathbb{R}^m$, sync interval $c$, current step $t$
\State Compute local keys and values: $Z \gets XW$
\State $Z_{\text{avg}} \gets \texttt{MeanToken}(Z)$
\State $\theta \gets \psi(Z_{\text{avg}})$ \Comment{Compute rotation param from local chunk}
\State $U \gets R(\theta)\, \bar{U}$ \Comment{Construct data-dependent subspace}
\State Compress: $Z_{\text{comp}} \gets Z U$ 
\State Broadcast $(Z_{\text{comp}}, \theta)$ to all other nodes
\State Receive $(Z_{(\text{comp},j)}, \theta_j)$ from all other nodes $j$
\ForAll{received $(Z_{(\text{comp},j)}, \theta_j)$}
    \State $U_j \gets R(\theta_j)\, \bar{U}$
    \State $Z_j \gets Z_{(\text{comp},j)}
    U_j^{\top}$ \Comment{Decompress}
\EndFor
\State Aggregate global $Z_g\in \{K_j,V_j \}, \forall j$ from all nodes
\State Compute blockwise attention: $A \gets \texttt{Softmax}(Q K^\top / \sqrt{d}) V$
\If{$t \bmod c = 0$}
    \State $W \gets \texttt{AllReduceAvg}(W)$ \Comment{Sparse sync of attention weights}
\EndIf
\end{algorithmic}
\end{algorithm}

\subsection{Unplugging the Projection Components}
\label{sec:unplug-projections}

Our method augments the transformer architecture with two non-standard components: (i) a small linear \emph{rotation head} predicting the rotation parameters \(\theta\), and (ii) low-rank \emph{projection matrices} \(U\) used for compressing activations. Although these components pose minimal overhead during training, strict API compatibility with off-the-shelf transformer models might be necessary for certain downstream applications.

As training proceeds, the learnable weights associated with our
auxiliary projection heads collapse onto the data–dependent subspaces
they steer.  Once the model is close to convergence we can therefore
\emph{drop these heads entirely}, reverting to a vanilla Transformer
without losing the predictive gains accumulated during training.
The following result formalizes the collapse mechanism.

\begin{informaltheorembox}

\textbf{Bound on “idle’’ attention directions with data dependent projectors.}\;
Let the  sample projector be
\(P(x)=U(x)U(x)^{\!\top}\).
Pick any other projector \(Q\) that projects onto an arbitrary subspace. 
Define the average overlap $p_Q\;:=\;\mathbb{E}_{x}\bigl[\|P(x)Q\|_{2}\bigr]\in[0,1]$. Run stochastic gradient descent with weight
decay~\(\lambda>0\).
Then, 
the attention weights that lie inside the \(Q\)-subspace obey
$
   \lim_{t \to \infty}   \bigl\|W^{(t)}Q\bigr\|_{F}
      \;\;\le\;\;
      \frac{p_Q\,L}{\lambda}
$
for an $L$ Lipshchitz bounded loss.
Hence, if the data almost never excites those directions (\(p_Q\!\ll\!1\)), 
the corresponding weights shrink away. That is, idle directions are pruned for
free. For the formal theorem and proof, see Theorem \ref{thm:subspace-decay-sgd} (Appendix).
\end{informaltheorembox}

Once the weights have collapsed onto their data–aligned subspaces, both the
rotation head and its basis matrix \(U\) are redundant.  We can therefore
\emph{detach} these components and perform a brief, low-learning-rate
fine-tuning pass to polish the remaining parameters.  
As Fig.~\ref{fig:unplug} shows, the loss curve remains smooth across this
transition, indicating that no \emph{optimisation shock} is introduced.

At inference time the model is now \emph{indistinguishable} from a standard
transformer: it adds \textbf{zero} extra parameters, requires no custom
kernels, and is fully compatible with existing deployment pipelines.

\section{Related Work}\label{sec:related}

\paragraph{Decentralized training.} 
Decentralized learning dispenses with a central coordinator, instead relying on a collective of \emph{autonomous} devices that cooperate over mesh-style networks to train large-scale models. These devices are typically heterogeneous and geographically dispersed, confronting links of nonuniform latency and bandwidth.  
The foundational theory on convergence and robustness has been established by~\cite{lian2017can,koloskova2019decentralized,koloskova2020unified}, while complementary systems work has demonstrated practical viability on real clusters~\cite{ryabinin2020towards,diskin2021distributed}. 
Most prior art, however, is confined to DDP settings~\cite{lian2017can,koloskova2019decentralized,koloskova2020unified,diskin2021distributed}, limiting model size to the aggregate memory of an individual node. Note that this is a comparatively well studied domain, and is orthogonal to the unexplored decentralized context parallel setting that we explore. A notable work in DDP domain is Power Gossip \cite{vogels2020powergossip}, which replaces synchronous all-to-all communication with gossip-style information exchange among neighboring replicas arranged in a mesh. Its key insight is that, when each replica trains independently via local SGD, the pairwise weight differences evolve in a low‑rank sub‑space, enabling them to efficiently compress the weight differences during gossip.  Another interesting DDP method is Photon \cite{sani2024photon},  where its communication savings stem primarily from infrequent gradient exchanges rather than from any explicit compression scheme. Such skip‑sync approaches are infeasible in context‑parallel pipelines, where activations must be transferred between nodes at every forward and backward pass. Nevertheless, these DDP‑style techniques are orthogonal to our method and could be combined with it in hybrid setups.

Scheduling-oriented approaches such as SWARM parallelism~\cite{ryabinin2023swarm} and Tasklets~\cite{yuan2022decentralized} alleviate straggler effects and network stochasticity, yet they still inherit the communication overhead intrinsic to the decentralized setting.  
In contrast, we introduce the first communication-compression strategy tailored to CP, removing a critical bandwidth bottleneck that has thus far hindered scaling decentralized models across larger context windows.

\paragraph{Context‑parallel attention.}
For single‑device long‑sequence processing, sparse approximations such as BigBird~\cite{zaheer2020big} cut attention complexity to $O(n)$, while IO‑aware exact kernels like FlashAttention~\cite{pagliardini2023faster} maximize hardware throughput with tiling and on‑chip caching. 
Recent systems research parallelizes the \emph{sequence dimension} itself \cite{liang2024torchtitan, rasley2020deepspeed, grattafiori2024llama}: Blockwise Parallel Transformers overlap compute and ring‑all‑reduce to achieve near‑linear speedups on sequences of $32\mathrm{K}$ tokens~\cite{liu2023blockwise}, and RingAttention extends the idea to virtually unlimited contexts via pipelined block exchanges~\cite{liu2024ringattention}.  
These methods, however, still broadcast full key/value tensors.  
Our approach instead transmits a compact low‑rank representation plus a lightweight rotation, reducing bandwidth while preserving exact attention semantics and thus complementing existing context‑parallel frameworks.






\section{Experiments}
\label{sec:experiments}

\begin{figure*}[ht]
    \centering
    \begin{subfigure}[b]{0.3\textwidth}
        \includegraphics[width=\textwidth]{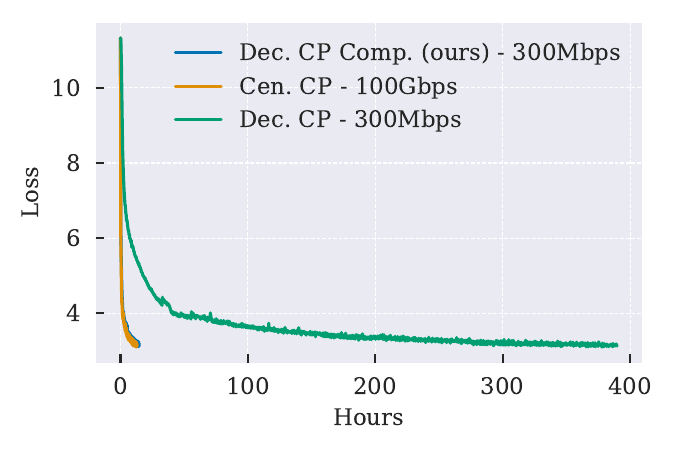}
 
        \label{fig:sub1}
    \end{subfigure}
    \hfill
    \begin{subfigure}[b]{0.3\textwidth}
        \includegraphics[width=\textwidth]{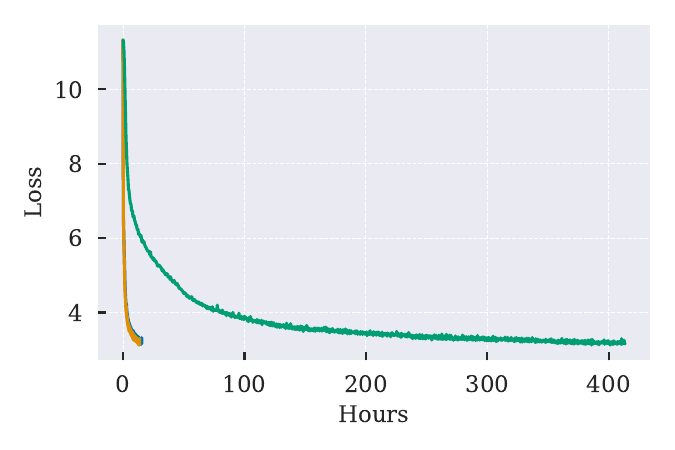}

        \label{fig:sub2}
    \end{subfigure}
    \hfill
    \begin{subfigure}[b]{0.3\textwidth}
        \includegraphics[width=\textwidth]{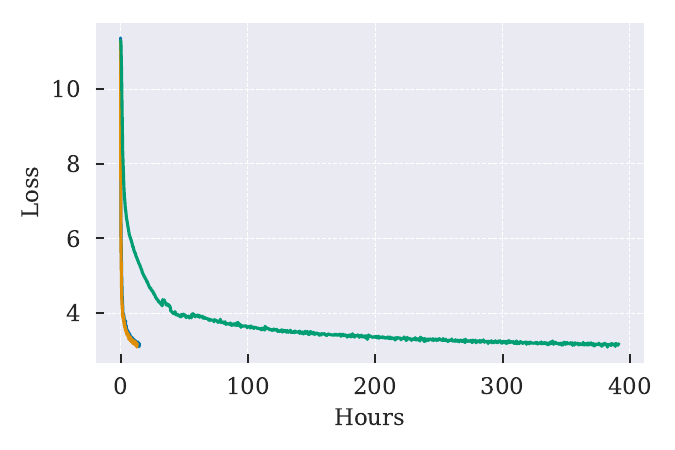}

        \label{fig:sub3}
    \end{subfigure}
  
    \caption{\textbf{Convergence in low-bandwidth settings.} From left to right: Fineweb, C4, and BookCorpus. The training curves are presented against wall-clock time for an $8$-layer ($800\mathrm{M}$) model trained with a $132\mathrm{K}$ context window parallelized across $8$ GPUs. Decentralized models utilize $300\mathrm{Mbps}$ connections while the centralized model has datacenter-grade $100\mathrm{Gbps}$ links. Our compressed model achieves on-par convergence to the centralized model, even under a $300\mathrm{Mbps}$ bandwidth budget. In contrast, the non-compressed decentralized model with $300\mathrm{Mbps}$ links suffers from significantly slower convergence.  }
    \label{fig:validation}
\end{figure*}

\subsection{Experimental Setup} 

We evaluate decoder-only models on three large-scale corpora -- FineWeb (FW)~\cite{wikitext},
C4~\cite{2019t5}, and
BookCorpus (BC)~\cite{bookcorpus}.
For each dataset, we reserve \(10\,\%\) of the training split for validation.
All model backbones follow \textsc{Llama 3}~\cite{dubey2024llama}; exact model specifications are given in the corresponding sections.
We use a \texttt{base-learning-rate} = \(3\times10^{-4}\) with linear warm-up and decay, and apply a \texttt{weight-decay} =  $0.01$.
Every transformer layer is compressed except for the final block, where \(K\) and \(V\) projections are compressed by \(98\,\%\) and \(95\,\%\) (overall $96.5\%$), respectively by choosing $r$ w.r.t. $d$ appropriately. We use the GPT2 tokenizer for all models.  

\begin{table}

\setlength{\tabcolsep}{3pt}  
\renewcommand{\arraystretch}{0.95}  
\centering
\caption{\textbf{Design ablations} (val. perplexity $\downarrow$). All models are trained for $10\mathrm{K}$ steps with a $132\mathrm{K}$ context. Second-order approximations preserve performance, while overcompressing $V$ degrades it.}
\label{tab:ablations}
\begin{tabular}{@{}lccc@{}}
\toprule
\textsc{\textbf{Setting}} & \textsc{\textbf{FW}} & \textsc{\textbf{C4}} & \textsc{\textbf{BC}} \\
\midrule
\textbf{Ours }                 & \textbf{22.64} & \textbf{23.33} & \textbf{25.27} \\
Ours + Fixed \(\bar{U}\)         & 26.57 & 27.11 & 30.33 \\
Ours + Rand. rot. \(R(\theta)\)      & 24.93 & 25.17 & 29.58 \\
Ours - 2nd-order approx.             & \textbf{22.64} & \textbf{23.33} & \textbf{25.27} \\
Ours - No warm start                 & 26.63 & 26.91 & 30.15 \\
Ours ($K,V \to 98\%$)         & 24.74 & 24.99 & 29.46 \\
Ours ($K \to 99\%, V \to 95\%$)      & 24.68 & 24.91 & 29.22 \\
\bottomrule
\end{tabular}

\end{table}
\subsection{Bandwidth efficiency in decentralized settings}

\begin{figure}[t]                       
  \centering
  %
  \begin{minipage}[t]{0.48\linewidth}
    \centering
    \includegraphics[width=\linewidth]{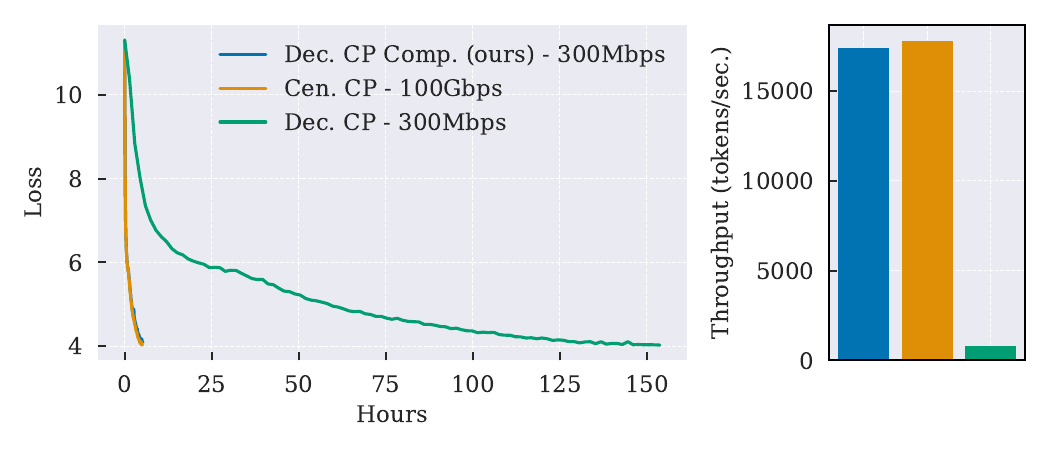}  
    \captionof{figure}{\textbf{Scaling across parallelism strategies.} Our compression based CP scheme can be seamlessly fused with other parallel training strategies. We train a $3\mathrm{B}$-parameter model ($32$ layers) with both pipeline parallel and CP enabled across $32$ $\mathrm{A}100$s. Our compressed approach yields substantial throughput gains over uncompressed decentralized CP and nearly matches the performance of centralized CP.
}
    \label{fig:scale}
  \end{minipage}
  \hfill                                   
  %
  \begin{minipage}[t]{0.48\linewidth}
    \centering
    \includegraphics[width=\linewidth]{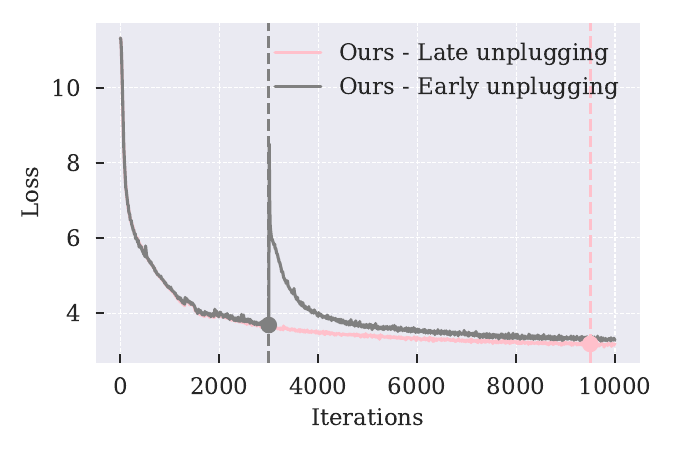}  
    \captionof{figure}{\textbf{Unplugging projection components.} After sufficient training, the rotation head and projection layers can be removed—reverting the network to a vanilla transformer—without impairing convergence. Training curves with dashed lines marking projection removal points. Late removal preserves convergence (see \S.~\ref{sec:unplug-projections}), while early removal causes a temporary disruption followed by surprisingly rapid recovery.}
    \label{fig:unplug}
  \end{minipage}
\end{figure}

We train an \(8\)-layer, \(800\mathrm{M}\)-parameter model (\texttt{embedding-size} = \(2048\), \texttt{attention-heads} = \(8\) ) under two network settings: a centralized \(100\mathrm{Gbps}\) fabric and  decentralized \(300\mathrm{Mbps}\) internet-grade links. Using CP, we process a sequence length of \(132\mathrm{K}\) tokens across eight $\mathrm{A100}$ GPUs connected at the respective bandwidths.
Fig.~\ref{fig:validation} shows that vanilla CP over a \(300\mathrm{Mbps}\) link is more than \( 20\times\) slower compared to a centralized \(100\mathrm{Gbps}\) mesh.  
With our compression, the same \(300\mathrm{Mbps}\) setup converges almost as fast as the centralized baseline.

\paragraph{Validation.} Table~\ref{tab:validation} reports test-time performance of the trained models. To this end,
we train each model up to its compute-optimal point, following the Chinchilla scaling law~\cite{hoffmanntraining}.
Specifically, for our \(800\mathrm{M}\)-parameter models, we use a $1:20$ model-to-token ratio and train for \(16\mathrm{B}\) tokens on each dataset. Remarkably, our compressed decentralized model matches, and even slightly outperforms, the perplexity of the centralized model at the same number of training iterations, while delivering significantly higher throughput than vanilla (uncompressed) CP over commodity links. Training the uncompressed model to completion over low-bandwidth links is computationally infeasible (estimated at over $150$ days), so we report only its throughput (TPS) in this setting.



\begin{table}[t]
\centering

\setlength{\tabcolsep}{4pt}
\renewcommand{\arraystretch}{0.95}
\caption{\textbf{Validation perplexity ($\downarrow$) and throughput (TPS).}
All models are trained with a $132\mathrm{K}$ context window to the compute-optimal point~\cite{hoffmanntraining} ($16\mathrm{B}$ training tokens). Our method yields a $20\times$ TPS boost while slightly outperforming centralized CP in perplexity, with minimal memory overhead.}
\label{tab:validation}
\begin{tabular}{@{}lccccr@{}}
\toprule
\textsc{\textbf{Model}} & \textsc{\textbf{FW}} & \textsc{\textbf{C4}} & \textsc{\textbf{BC}} & \textsc{\textbf{TPS}} & \textsc{\textbf{Mem (GB)/GPU}} \\
\midrule
Cen. CP - 100Gbps     & 17.18 & 17.51 & 17.88 & 56K & 38.4 \\
Dec. CP - 300Mbps\textsuperscript{\dag} & -- & -- & -- & 2.7K & 38.4 \\
\textbf{Dec. CP Comp. - 300Mbps (ours)} & \textbf{17.06} & \textbf{17.47} & \textbf{17.81} & \textbf{55K} ($\times 20$) & 38.7 (+0.7\%) \\
\bottomrule
\end{tabular}

\raggedright
\textsuperscript{\dag} Training uncompressed models to convergence at $300\mathrm{Mbps}$ is infeasible ($>$150 days); only throughput is reported.
\end{table}



\subsection{Ablations}

We perform ablations on $800\mathrm{M}$-parameter models with a $132\mathrm{K}$ context across eight $\mathrm{A}100$ GPUs (see Table~\ref{tab:ablations}). Models using learned rotations outperform those with fixed or random projections. The second-order exponential approximation does not impact performance, confirming its adequacy. Omitting the warm-start initialization of principal directions (\(\bar{U}\)) noticeably degrades results, highlighting the importance of this prior. 
\begin{table}

\setlength{\tabcolsep}{3pt}  
\renewcommand{\arraystretch}{0.95}  
\centering
\caption{\textbf{Effect of warmup steps} (val. perplexity $\downarrow$). All models are trained for $10\mathrm{K}$ steps with a $132\mathrm{K}$ context. The method is not highly sensitive to the number of warmup steps.}
\label{tab:warmup}
\begin{tabular}{@{}lc@{}}
\toprule
\textsc{\textbf{Warmup-steps}} & \textsc{\textbf{Perplexity}}\\
\midrule
0 & 26.63 \\
100 &	24.44 \\
300	& 22.66 \\
500	& 22.64 \\
1000 & 22.87 \\
2000 & 22.64 \\
5000 & 22.71 \\
\bottomrule
\end{tabular}

\end{table}

\textbf{Scaling:} Our compression based CP scales well and can be seamlessly fusing with other parallel training strategies. We scale the model to $32$ layers ($3\mathrm{B}$ parameters) with both pipeline parallel and CP enabled over $32$ $\mathrm{A}100$s (Fig.~\ref{fig:scale}) and achieve a significant throughput gain.

\textbf{Reparameterization:} A key step of our method is reparameterizing \(U\) which bypasses expensive Riemannian operations (QR/SVD pullbacks). As shown in Table~\ref{tab:ablations_tps}, this reparameterization significantly improves throughput (TPS). More ablations against   architecture choices are provided in Appendix~\ref{app:ablations}.

\textbf{Warmup steps:} To measure the effect warmup steps of we conducted an ablation study varying the warm-up duration and evaluated the resulting perplexity on the FineWeb dataset. The results are shown in Table \ref{tab:warmup}. As demonstrated, even with a reduced warm-up of 300 steps, the model achieves comparable performance, indicating no significant degradation. In practice, we default to 500 steps to provide a safe and stable baseline. This study further emphasizes the lightweight and robust nature of our warm-up strategy, especially in contrast to the more elaborate scheduling mechanisms commonly employed in modern LLM pre-training. Note that the perplexity differences are minor and stable, indicating performance is stable after 300 warm-up steps.

\subsection{Unplugging Projections and Rotation Heads}

As discussed in \S.~\ref{sec:unplug-projections}, practitioners may prefer reverting to a standard transformer after pretraining for compatibility with downstream frameworks. We empirically validate our theoretical prediction that attention weights progressively align with the projection subspace, allowing safe removal of projection layers and rotation heads near the end of training. Fig.~\ref{fig:unplug} shows that removing these components late preserves convergence, while doing so prematurely disrupts training.





\subsection{Comparison Against Baselines}

As no prior baselines exist for CP compression, we construct two:  
(i)~\textbf{Sparsification}—a Top-10\% scheme (90\% compression), transmitting only the largest-magnitude entries of the \(K,V\) chunks, inspired by common DDP compression methods;  
(ii)~\textbf{Quantization}—a 4-bit quantization (75\% compression) of the \(K,V\) activations prior to transmission, following standard practices in activation compression. As shown in Fig.~\ref{fig:baselines} (left), we outperform these baselines comprehensively ($132$k context window) even when using a more aggressive compression rate of $96.5\%$.

\begin{figure}  
  \centering
  \begin{subfigure}{0.49\linewidth}
    \includegraphics[width=\linewidth]{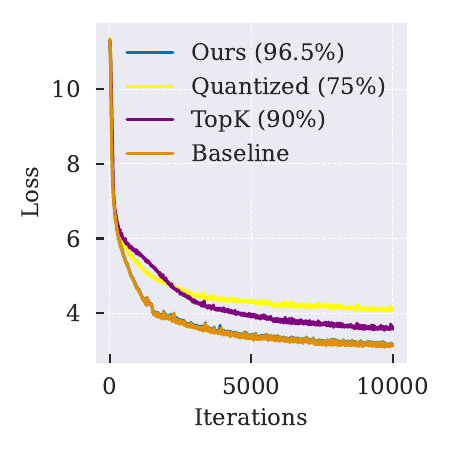}
    \label{fig:wrap-left}
  \end{subfigure}\hfill
  \begin{subfigure}{0.49\linewidth}
    \includegraphics[width=\linewidth]{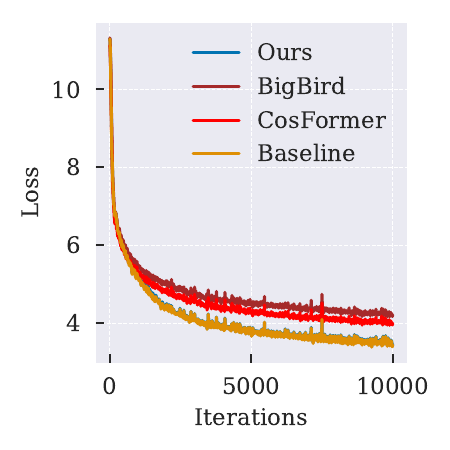}
    \label{fig:wrap-right}
  \end{subfigure}
 \
  \caption{\textbf{Baseline comparisons.}
\textit{Left:} Because no method currently compresses context-parallel training, we build two baselines; Top-$k$ sparsification and quantization.
\textit{Right:} We also compare with long-context models BigBird and CosFormer. Both are limited to $32\mathrm{K}$ tokens on $\mathrm{A}100$ GPUs, so all models are evaluated at that length. In both panels, our compressed CP curve is nearly indistinguishable from the uncompressed reference, whereas every baseline falls well short.}
  \label{fig:baselines}
 
\end{figure}

For completeness, we also compare against long-context models BigBird~\cite{zaheer2020big} and CosFormer~\cite{qin2022cosformer}, which are not designed for CP and can handle at most $32\mathrm{K}$ tokens on an $\mathrm{A}100$. For a fair comparison, we apply our compression to CP across four GPUs, each processing $8\mathrm{K}$ tokens. As shown in Fig.~\ref{fig:baselines} (right), both baselines exhibit significantly worse convergence than our method. All experiments are performed on $800\mathrm{M}$ parameter models.

\section{Conclusion}

We propose the first compression method that enables context-parallel training of language models in decentralized environments with low-bandwidth interconnects. Our approach supports training with context lengths over $100\mathrm{K}$ tokens on isolated GPUs connected via internet-grade links (e.g., $300\mathrm{Mbps}$), while matching the wall-clock convergence of centralized systems with high-speed ($100\mathrm{Gbps}$) connections. Additionally, our method preserves compatibility with standard transformer architectures by allowing the projection layers to be removed after training, facilitating seamless deployment in downstream frameworks. We provide a theoretical analysis of the key properties of our method and validate its effectiveness through an extensive empirical evaluation.

\begin{table}
\centering
\setlength{\tabcolsep}{4pt}
\renewcommand{\arraystretch}{0.9}
\caption{\textbf{TPS gain from design choices.} Reparameterization and second-order approximation yield significant throughput improvements.}
\label{tab:ablations_tps}
\begin{tabular}{@{}l c@{}}
\toprule
\textsc{\textbf{Setting}} & \textsc{\textbf{TPS}} ($\uparrow$) \\
\midrule
\textbf{Ours }               & \textbf{55K} \\
w/o reparam.        & 37K \\
w/o $2^{nd}$ ord. approx.         & 30K \\
\bottomrule
\end{tabular}
\end{table}

\section{Limitations}
\label{sec:limitations}

Our compression method delivers near-lossless convergence in context-parallel training, but several open questions remain.  
First, alternative reparameterisations beyond simple subspace rotations may unlock further accuracy or efficiency gains.  
Second, the method’s surprising ability to locate good minima even as the search space is heavily reduced (via very low-dimensional \(\theta\)) lacks a rigorous explanation; its ties to recent work on implicit regularisation and lottery-ticket-style phenomena deserve closer study.  Despite these gaps, this work establishes the first baseline for context-parallel compression and we hope it spurs deeper theoretical and empirical exploration.









\bibliographystyle{plainnat}
\nobibliography*
\bibliography{main}

\newpage

\appendix

\begin{center}
    {\LARGE \textbf{Appendix: Mixtures of Subspaces for Bandwidth Efficient Context Parallel Training}}
\end{center}

\section{Theoretical analysis}
To provide formal support for our proposed joint optimization of the factorization \( W = B U U^\top \) discussed in the main paper, we now present a rigorous  analysis. Specifically, we consider the optimization dynamics on the product manifold \( \mathcal{M} = \mathbb{R}^{d\times r} \times \mathrm{St}(d,r) \), where \(B\) is optimized in standard Euclidean space and \(U\) resides on the Stiefel manifold. Under standard smoothness and boundedness assumptions commonly adopted in optimization theory, the following theorem establishes that our gradient descent updates converge linearly to a first-order stationary point.

\label{app:theory}
\begin{lemma}
\label{thm:ProdManifoldConvergenceFull}
Let $F:\mathbb{R}^{n\times n}\!\times\!\mathbb{R}^{p}\!\to\!\mathbb{R}$ be $L$-smooth and bounded below on rank-$r$ matrices.  Write each attention weight as $W=BUU^\top$ with $B\in\mathbb{R}^{n\times n}$ and $U\in\mathrm{St}(n,r):=\{X\in\mathbb{R}^{n\times r}\mid X^\top X=I_r\}$, and collect the remaining parameters in $\vartheta\in\mathbb{R}^{p}$.  Define $\Phi(B,U,\vartheta):=F(BUU^\top,\vartheta)$.  Suppose $\Phi$ satisfies a PL inequality with constant $\mu>0$ in a neighbourhood of an optimum $(B_\star,U_\star,\vartheta_\star)$ and that $(B_0,U_0,\vartheta_0)$ lies in this neighbourhood.  Using the updates $B_{k+1}=B_k-\alpha_B\nabla_B\Phi_k$, $U_{k+1}=\mathrm{exp}_{U_k}(-\alpha_U\mathrm{grad}_U\Phi_k)$, and $\vartheta_{k+1}=\vartheta_k-\alpha_\theta\nabla_\theta\Phi_k$ with fixed stepsizes $0<\alpha_B,\alpha_U,\alpha_\theta\le 1/L$ and $\alpha:=\min\{\alpha_B,\alpha_U,\alpha_\theta\}$, we have for all $k\ge0$
\[
\Phi(B_k,U_k,\vartheta_k)-\Phi_\star\le(1-\alpha\mu)^k\bigl[\Phi(B_0,U_0,\vartheta_0)-\Phi_\star\bigr],\qquad\Phi_\star:=\Phi(B_\star,U_\star,\vartheta_\star),
\]
and the iterates converge linearly to $(B_\infty,U_\infty,\vartheta_\infty)$ with $B_\infty U_\infty U_\infty^\top=W_\star$ and $\vartheta_\infty=\vartheta_\star$.
\end{lemma}

\begin{proof}
Because $F$ is $L$-smooth and $B\!\mapsto\!BUU^\top$ is linear,  
\[
\nabla_{B}\Phi(B,U,\vartheta)=\bigl[\nabla_WF(W,\vartheta)\bigr]_{W=BUU^\top}UU^\top.
\]
And we also know $\|U\|$ is bounded so $\nabla_{B}\Phi(B,U,\vartheta)$ is $L$-Lipschitz in $(B,U,\vartheta)$; the same holds for $\nabla_\vartheta\Phi$.  
For $U\in\mathrm{St}(n,r)$ the Riemannian gradient is  
\[
\mathrm{grad}_U\Phi=\pi_{T_U\mathrm{St}(n,r)}\bigl(\nabla_U\Phi\bigr),
\]
Where $T_U$ is the tangent space and $\pi$ is the projection back on to $\mathrm{St}(n,r)$.  See that the projection $\pi_{T_U\mathrm{St}(n,r)}$ is bounded, so $\mathrm{grad}_U\Phi$ is also $L$-Lipschitz.

\smallskip
Now we consider gradient descent in each parameter block.
By the Euclidean descent lemma \cite{NocedalWright2006},
\[
\Phi(B_{k+1},U_k,\theta_k)\le\Phi_k-\tfrac{\alpha_B}{2}\|\nabla_B\Phi_k\|^2.\tag{1}
\]
With $B_{k+1},\theta_k$ fixed, the Riemannian descent lemma on $\mathrm{St}(n,r)$ \cite{AbsilMahonySepulchre2008} gives
\[
\Phi(B_{k+1},U_{k+1},\vartheta_k)\le\Phi(B_{k+1},U_k,\vartheta_k)-\tfrac{\alpha_U}{2}\|\nabla_U\Phi_k\|^2.\tag{2}
\]
Finally, keeping $(B_{k+1},U_{k+1})$ fixed,
\[
\Phi_{k+1}\le\Phi(B_{k+1},U_{k+1},\vartheta_k)-\tfrac{\alpha_\theta}{2}\|\nabla_\vartheta\Phi_k\|^2.\tag{3}
\]

\smallskip

Summing (1)–(3) yields
\[
\Phi_{k+1}\le\Phi_k-\tfrac12\!\bigl[\alpha_B\|\nabla_B\Phi_k\|^2+\alpha_U\|\nabla_U\Phi_k\|^2+\alpha_\theta\|\nabla_\theta\Phi_k\|^2\bigr].
\tag{4}
\]
Let $\alpha:=\min\{\alpha_B,\alpha_U,\alpha_\theta\}$.

\smallskip

By PL assumption, there exist $\mu>0$ and a neighbourhood $\mathcal{N}\subseteq\mathbb{R}^{n\times n}\!\times\!\mathrm{St}(n,r)\!\times\!\mathbb{R}^{p}$ containing the iterates such that
\[
\|\nabla_B\Phi\|^2+\|\nabla_U\Phi\|^2+\|\nabla_\theta\Phi\|^2\ge2\mu\bigl[\Phi-\Phi_\star\bigr]\;\;\forall(B,U,\vartheta)\in\mathcal N.\tag{5}
\]

\smallskip

Combining (4) and (5) gives
\[
\Phi_{k+1}-\Phi_\star\le(1-\alpha\mu)\bigl[\Phi_k-\Phi_\star\bigr],
\]
and induction yields

\[
  \Phi_k-\Phi_\star\le(1-\alpha\mu)^k\bigl[\Phi_0-\Phi_\star\bigr].
\]

which is the claimed linear rate.

\smallskip

The geometric decay plus $L$-smoothness implies
$\sum_k\|\nabla_B\Phi_k\|^2<\infty$ (similarly for $\nabla_U\Phi_k$ and
$\nabla_\theta\Phi_k$), hence all gradient blocks vanish.  Any limit
point $(B_\infty,U_\infty,\theta_\infty)$ therefore satisfies the
first-order conditions and attains $\Phi_\star$.  Consequently
$W_k:=B_kU_kU_k^\top$ converges to
$W_\star:=B_\infty U_\infty U_\infty^\top$, and
$\theta_k\to\theta_\infty=\theta_\star$.

\end{proof}

\subsection{Geometric Impact of Rotational Reparameterisation}
\label{sec:geometry_reparam}

Next, to formally verify that our reparameterization \(U(\theta)=R(\theta)\,\overline{U}\) preserves the  stationary points of the optimization landscape, we present the complete theorem and proof in this appendix. Specifically, we rigorously demonstrate that optimizing in the unconstrained parameter space \(\theta\) via ordinary SGD or Adam does not alter the nature or locations of local minima and strict saddle points of the original constrained optimization problem. This formalizes and extends the intuitive reasoning discussed in the main text, confirming that our rotational reparameterization is geometrically faithful and optimization-efficient.

\begin{theorem}
\label{thm:reparam-U}

Consider $\Phi(B,U,\vartheta)$ as defined in Proposition \ref{thm:ProdManifoldConvergenceFull}. and fix an orthonormal matrix $\overline U\in \mathrm{St}(d,r)$. Suppose $\theta\mapsto R(\theta)$ is surjective (or dense) in $O(d) = \{R \in \mathbb{R}^{d \times d} | R^TR = \mathbf{I}_d\}$. Define the Euclidean reparameterisation
\[
   \widehat{\Phi}(B,\theta,\vartheta)
      := \Phi\!\bigl(B ,R(\theta)\,\overline U\overline U^{\!\top}
                           R(\theta)^{\!\top},\;\vartheta\bigr).
\]
Then the sets of local minima of $\widehat{\Phi}(B,\theta,\vartheta)$ and
 $\Phi(B,U,\vartheta)$ coincide.
 \end{theorem}

 \begin{proof}
Denote by $\phi(B,\theta, \vartheta) = (B, U,\;\vartheta )$ the \emph{lifting map} from the
re-parameterised domain 
$\mathcal D_1=\mathbb R^{d\times d}\times\Theta$
to the original domain 
$\mathcal D_2=\mathbb R^{d\times d}\times\mathrm{St}(d,r)$, where $\theta \in \Theta$.

For any $U\in\mathrm{St}(d,r)$ extend its columns to an orthogonal matrix
$R\in O(d)$ with $UR^\top=\overline U$; surjectivity of $\theta\!\mapsto\!R(\theta)$
(or its denseness plus continuity) yields $\theta$ with $R(\theta)=R$,
so $U=R(\theta)\overline U$.  
Hence $\phi(\mathcal D_1)=\mathcal D_2$. That is, $\phi$ is continuous and surjective.

Now, let $(B^\star,\theta^\star,\vartheta^*)$ be a local minimiser of $\widehat{\Phi}$ and set
$U^\star:=R(\theta^\star)\overline U=\phi(B^\star,\theta^\star)_U$.
Assume, towards a contradiction, that $(B^\star,U^\star, \vartheta^*)$ is
\emph{not} a local minimiser of $\Phi$.
Then there exists a sequence $(B_k,U_k)\!\to\!(B^\star,U^\star)$ with
$\Phi(B_k,U_k, \vartheta_k)\!<\!\Phi(B^\star,U^\star, \vartheta^*)$.
By surjectivity (or density) pick
$\theta_k$ such that $U_k=R(\theta_k)\overline U$ and
$\theta_k\!\to\!\theta^\star$.
Continuity of $\widehat{\Phi} = \Phi \circ \phi$ gives
\(
\widehat{\Phi}(B_k,\theta_k, \vartheta_k)=\Phi(B_k,U_k,\vartheta_k)
               < \Phi(B^\star,U^\star,\vartheta^*)
               = \widehat{\Phi}(B^\star,\theta^\star, \vartheta^*),
\)
contradicting local minimality of $\widehat{\Phi}$. Therefore, local minima of $\hat{\Phi}$ lift to local minima of $\Phi$.
Let $(B^\star,\theta^\star)$

Conversely, let $(\widehat B,\widehat U, \widehat{\vartheta})$ be a local minimiser of
$\Phi$.
Choose any $R\in O(d)$ with $\widehat U=R\,\overline U$ and
pick $\widehat\theta$ such that $R(\widehat\theta)=R$.
If $(B,\theta, \vartheta)\!\to\!(\widehat B,\widehat\theta, \widehat{\vartheta})$ then
$\phi(B,\theta, \vartheta)\!\to\!(\widehat B,\widehat U, \widetilde{\vartheta})$, hence
\(
\widehat{\Phi}(B,\theta, \vartheta)=\Phi(B,U,\vartheta)
           \ge\Phi(\widehat B,\widehat U, \widehat{\vartheta})
           =\widehat{\Phi}(\widehat B,\widehat\theta, \widehat{\vartheta}),
\)
so $(\widehat B,\widehat\theta, \widehat{\vartheta})$ is a local minimiser of $\widehat{\Phi}$.

Thus, $\phi$ induces a bijection between the sets of
local minima of $\Phi$ and $\widehat{\Phi}$; hence the two
optimisation problems have exactly the same local minima.
\end{proof}

Theorem~\ref{thm:reparam-U} proves that the Euclidean
reparameterisation
\(
   U(\theta)=R(\theta)\,\overline U,\;
   R(\theta)\in O(d)
\)
preserves \emph{stationary points}.
In this section we take a deeper look at how the transformation
affects the \emph{surrounding loss landscape}, focusing on curvature,
conditioning, and global distortions.
Specifically, we show that by constraining  $\|\theta\|$ to a moderate range (via clipping), one can make sure that the curvature of the re-parameterized landscape remain similar to that of the original objective. Throughout, $\Phi(B,U,\vartheta)$ denotes the original objective,
while
\(
  \widehat\Phi(B,\theta,\vartheta)
  :=\Phi\!\bigl(B,R(\theta)\,\overline U,\vartheta\bigr)
\)
is its pull-back to the unconstrained parameters~$\theta$.

\subsubsection{Hessian pull–back formula}
Let
\(
   g_U = \nabla_U\Phi\in\R^{d\times r},
   \qquad
   H_U = \nabla_{UU}^2\Phi\in\R^{dr\times dr}
\)
be the gradient and intrinsic (Euclidean) Hessian \emph{w.r.t. the
Stiefel coordinates}.
Define the \emph{Jacobian} of the reparameterisation
\(
   J(\theta)=\dfrac{\partial U(\theta)}{\partial\theta}
            \in\R^{dr\times m},
   \quad m=\dim\Theta .
\)

\begin{lemma}
\label{lem:hessian_pullback}
For any $\theta$ the Euclidean Hessian of $\widehat\Phi$ satisfies
\begin{equation}
\label{eq:hess_pull}
   H_\theta
   :=\nabla^2_{\theta\theta}\widehat\Phi
   \;=\;
   J^\top H_U\,J
   \;+\;
   \sum_{i=1}^{m}
     \bigl(\partial_{\theta_i}J\bigr)^\top g_U .
\end{equation}
At \emph{stationary points}---i.e.\ when $g_U=0$---the second term
vanishes and
\begin{equation}
\label{eq:hess_congruence}
      H_\theta \;=\; J^\top H_U\,J .
\end{equation}
\end{lemma}

\begin{proof}[Proof of Lemma~\ref{lem:hessian_pullback}]
Fix $(B,\theta,\vartheta)$ and denote\footnote{%
For clarity we \emph{vectorise} matrices, writing
$u=\operatorname{vec}(U)\in\R^{dr}$,
$g_U=\nabla_{u}\Phi\in\R^{dr}$,
and
$H_U=\nabla^{2}_{uu}\Phi\in\R^{dr\times dr}$.%
}
\(
   u(\theta)=\operatorname{vec}\!\bigl(U(\theta)\bigr),\;
   g_U(u)=\nabla_u\Phi,\;
   J(\theta)=\dfrac{\partial u}{\partial\theta}\in\R^{dr\times m}.
\)

By the multivariate chain rule
\[
   \nabla_\theta\widehat\Phi
   \;=\;
   J^{\!\top}\,g_U .
\]

Differentiate once more:
\[
   H_\theta
   \;=\;
   \nabla_\theta\!\bigl(J^{\!\top}g_U\bigr)
   \;=\;
   \bigl(\nabla_\theta J^{\!\top}\bigr)g_U
   \;+\;
   J^{\!\top}\!\bigl(\nabla_\theta g_U\bigr).
\]
Compute the two addends separately.

\smallskip
\noindent\emph{(i) Derivative of the Jacobian.}
Write the $i$-th column of $J$ as $J_{(:,i)}$.
Then \(
   (\nabla_\theta J^{\!\top})g_U
   =\sum_{i=1}^{m}\!
      \bigl(\partial_{\theta_i}J\bigr)^{\!\top} g_U .
\)

\smallskip
\noindent\emph{(ii) Derivative of the pulled-back gradient.}
Because \(g_U\) is a function of \(u(\theta)\),
\[
   \nabla_\theta g_U
   \;=\;
   \bigl(\nabla_u g_U\bigr)\,\nabla_\theta u
   \;=\;
   H_U\,J .
\]
Substituting (i) and (ii) yields
\[
   H_\theta
   \;=\;
   J^{\!\top} H_U\,J
   \;+\;
   \sum_{i=1}^{m}
     \bigl(\partial_{\theta_i}J\bigr)^{\!\top} g_U ,
\]
which is precisely \eqref{eq:hess_pull}.  
If $\theta$ is a stationary point of $\widehat\Phi$,
then $g_U=0$ and the second term vanishes, giving
\eqref{eq:hess_congruence}.
\end{proof}

\paragraph{Spectral consequences.}
Writing the singular values of $J$ as
\(
  \sigma_1\ge\sigma_2\ge\!\dotsb\!\ge\sigma_r>0
\)
and the ordered eigenvalues of $H_U$ as
\(
  \lambda_1\ge\dotsb\ge\lambda_r,
\)
classic congruence inequalities yield
\begin{equation}
\label{eq:eigen_bounds}
  \sigma_r^{\,2}\,\lambda_r
  \;\le\;
  \lambda_{\min}(H_\theta)
  \;\le\;
  \lambda_{\max}(H_\theta)
  \;\le\;
  \sigma_1^{\,2}\,\lambda_1 .
\end{equation}
Hence the \textbf{condition number} transforms as
\(
  \kappa(H_\theta)
  \le
  \bigl(\sigma_1/\sigma_r\bigr)^{2}\,\kappa(H_U).
\)
Good scaling of~$J$ is therefore essential for maintaining a
well-conditioned landscape.

\subsubsection{Local properties with the exponential map}
We instantiate $R(\theta)$ with the exponential map
\begin{equation}
\label{eq:exp_map}
   R(\theta)
   \;=\;
   \exp\Bigl(\sum_{i=1}^{m}\theta_i A_i\Bigr),
   \qquad
   A_i^\top=-A_i,
   \quad
   \langle A_i,A_j\rangle_F=\delta_{ij},
\end{equation}
where $\{A_i\}$ is an \emph{orthonormal basis} of the Lie algebra
$\mathfrak{o}(d)$.
For $\|\theta\|\ll1$ the first-order expansion gives
\[
   R(\theta)=I+\!\sum_i\theta_iA_i+\mathcal O\!\bigl(\|\theta\|^{2}\bigr),
   \quad
   J_i
     :=\frac{\partial U}{\partial\theta_i}
     =A_i\,\overline U+\mathcal 
 O\!\bigl(\|\theta\|\bigr).
\]
Because the $A_i\overline U$ are \emph{orthonormal} in the embedded
space $\R^{d\times r}$, the Gram matrix satisfies
\(J^\top J = I + \mathcal  O(\|\theta\|)\).
Hence near $\theta=0$,
\[
   \sigma_1,\sigma_r = 1 + \mathcal 
 O(\|\theta\|),
   \quad
   \kappa(J)=1+ \mathcal  O(\|\theta\|),
\]
so by~\eqref{eq:eigen_bounds} the condition number of $H_\theta$
matches that of $H_U$ up to first order.  
\emph{Locally}, therefore, the exponential map is almost
\emph{isometric}: plain SGD or Adam on~$\theta$ senses essentially the
same curvature as a Riemannian optimiser on~$U$.

\begin{corollary}
\label{cor:local_cond}
Fix any radius $\rho<\pi$ and let
\(
  \mathcal B_\rho=\{\theta\in\Theta \mid \|\theta\|_2\le\rho\}.
\)
There exists $c_\rho>0$ such that for all
$\theta\in\mathcal B_\rho$ and stationary $(B,\theta,\vartheta)$,
\[
   \frac{1}{1+c_\rho}\,H_U
   \;\preceq\;
   H_\theta
   \;\preceq\;
   (1+c_\rho)\,H_U .
\]
In particular,
\(
  \kappa(H_\theta)\le(1+c_\rho)^2\,\kappa(H_U).
\)
\end{corollary}

\begin{proof}
Combine~\eqref{eq:hess_congruence} with the
bound $\|J^\top J-I\|_2\le c_\rho$ that follows from the
expansion above and continuity of $R(\theta)$ inside $\mathcal B_\rho$.
\end{proof}

\paragraph{Intuition.}
Corollary~\ref{cor:local_cond} asserts that, as long as the Lie–algebra
parameter stays inside the ball
\(\mathcal B_\rho=\{\theta:\|\theta\|_2\le\rho\}\) with \(\rho<\pi\), the
Hessian in the new coordinates, \(H_\theta\), differs from the original
Stiefel-space Hessian, \(H_U\), by no more than a scalar factor
\(1+c_\rho\) in either direction:
\[
   \frac{1}{1+c_\rho}\,H_U
   \;\preceq\;
   H_\theta
   \;\preceq\;
   (1+c_\rho)\,H_U,
   \qquad   c_\rho=\mathcal O(\rho).
\]
Here “\(\preceq\)” denotes the usual Loewner order, so the inequality
means that every quadratic form \(v^{\!\top}H_\theta v\) lies between
\(v^{\!\top}H_U v /(1+c_\rho)\) and
\((1+c_\rho)\,v^{\!\top}H_U v\).  Consequently the condition number is
inflated by at most the square of this factor,
\(\kappa(H_\theta)\le (1+c_\rho)^2\kappa(H_U)\),
and tends back to \(\kappa(H_U)\) as \(\rho\to0\).  In practical terms,
when the rotation angles encoded by \(\theta\) remain moderate
(\(\|\theta\|\lesssim 1\) rad), plain Euclidean optimisers experience
almost the same curvature, and therefore the same step-size stability
and convergence speed, as a Riemannian optimiser that works directly on
\(U\).  Only as \(\|\theta\|\) approaches the injectivity radius
(\(\approx\pi\)) does the Jacobian cease to be nearly orthonormal,
distorting the landscape enough to break this near-isometry. \textbf{In practice, we clip $\| \theta\| < 0.5$ to enforce this constraint.} Note that without this constraint, the models sometimes demonstrated unstable convergence.

\subsubsection{Global distortions and injectivity radius}
The exponential map is injective on
\(
  \|\theta\|_2<\pi\sqrt{2}
\)
(the minimal distance to the cut-locus).
As $\|\theta\|\!\to\!\pi$, several phenomena occur: \textbf{a) Jacobian degeneration.}Some singular values $\sigma_i(J)$
      collapse to zero, flattening curvature along their directions and
      introducing plateaus in~$\widehat\Phi$ even if $H_U$ is full-rank.
\textbf{b) Anisotropic stretching.}
      Other singular values \emph{blow up},
      turning moderate curvature in $H_U$ into steep walls in
      $\theta$-space.

\subsubsection{Practical takeaway}
Rotational reparameterisation via the exponential map yields a
\emph{locally} isometric embedding of the Stiefel manifold into
Euclidean space, ensuring that first-order methods experience
essentially unchanged curvature near the solution set.  
However, as the Lie-algebra coordinates move towards the injectivity
radius, the Jacobian may \emph{distort} the landscape dramatically.
Clipping or norm-projection of~$\theta$ acts like a trust-region
mechanism that retains favourable conditioning but risks biasing the
search if the feasible clip radius is chosen too small.

The practitioner’s rule-of-thumb is therefore:

\begin{informaltheorembox}
\emph{Maintain $\|\theta\|_2\lesssim1\text{ rad}$ whenever possible; if
larger rotations are essential, monitor $\|J\|_2$ or curvature
statistics and re-centre / rescale when they inflate or collapse.}
\end{informaltheorembox}

\subsection{Rank collapse of the attention weights}

To rigorously characterize the collapse of auxiliary projection heads onto data-dependent subspaces, as discussed in the main text, we present a detailed result and proof next. Specifically, we demonstrate formally how attention weights associated with infrequently activated directions shrink under optimization with weight decay. This confirms analytically that auxiliary projection heads become effectively redundant near convergence, justifying their safe removal and transition to a standard Transformer architecture without sacrificing accumulated predictive performance. First, we consider fixed projection matrices below. That is the case where $U$ is not data dependent.

\begin{proposition}
\label{thm:confine-to-U}
Let 
$
    U\,\in\,\mathbb{R}^{d\times r}, 
    \qquad U^{\top}U = I_r,
    \qquad 
    P := UU^{\top},\;
    P_\perp := I_d - P .
$
Fix any \(\ell_2\)-regularised objective of the form  
$
    \mathcal L(W) 
    \;=\;
    F\!\bigl(WP\bigr) 
    \;+\; 
    \frac{\lambda}{2}\,\|W\|_F^{2},
    \qquad W\in\mathbb{R}^{d\times k},\;
           \lambda>0.
$
Assume only that \(F:\mathbb{R}^{d\times k}\!\to\!\R\) is continuously differentiable.  
Consider the Gradient flow  \(\displaystyle \dot W(t) = -\nabla_W\mathcal L\bigl(W(t)\bigr)\),
Then the \emph{orthogonal component}
\(W_{\perp}:=WP_\perp\) obeys
\[
        \|W_{\perp}(t)\|_F
        \;=\;
           (1-\eta\lambda)^{\,t}\,\|W_{\perp}^{(0)}\|_F
\]

Hence \(W_{\perp}(t)\!\to\!0\), and every limit point satisfies  
\(W^\star \in \operatorname{col}(U)\).
\end{proposition}

\begin{proof}
Let \(G := \nabla_XF(X)\bigl|_{X=WP}\in\R^{d\times k}\).
Then we have,
\[
    \nabla_W F(WP) 
    \;=\; 
    G\,P .
    \tag{1}
\]

\smallskip
Consider the full gradient of the regularized objective
\[
    \nabla_W\mathcal L(W) \;=\; G\,P + \lambda W .
    \tag{2}
\]

\smallskip

Define the parallel and orthogonal blocks
\(
    W_{\parallel}:=WP,\;
    W_{\perp}:=WP_\perp
\)
so that \(W = W_{\parallel}+W_{\perp}\).
Because \(PP_\perp = 0\),
\[
    W_{\parallel}P_\perp = 0,
    \quad 
    W_{\perp}P = 0 .
    \tag{3}
\]

\smallskip


      Using $\nabla_W\mathcal L(W) \;=\; G\,P + \lambda W$ in the update and projecting:
      \[
          W^{(t+1)}P_\perp
          \;=\;
          \bigl(
            W^{(t)} 
            - \eta (GP + \lambda W^{(t)})
          \bigr) P_\perp
          = (1-\eta\lambda)\,W^{(t)}P_\perp .
      \]
      Induction yields  
      \(\displaystyle 
         \|W_{\perp}^{(t)}\|_F
         = (1-\eta\lambda)^{t}\|W_{\perp}^{(0)}\|_F .
      \)

\smallskip
So we have \(\|W_{\perp}(t)\|_F\to0\).
Thus \(P_\perp W(t)\to 0\), i.e.\ every accumulation point
lies entirely in \(\operatorname{col}(U)\).  
\end{proof}

Next, we consider the case where $U$ is data dependent under full-batch gradient descent.

\begin{theorem}
\label{thm:subspace-decay-clean}
For every sample \(x\) let  
$
   U(x)\in\mathbb{R}^{d\times r},
   \qquad
   U(x)^{\top}U(x)=I_r,
   \qquad
   P(x):=U(x)U(x)^{\top}\in\mathbb{R}^{d\times d}.
$
Fix a loss family \(F_x:\mathbb{R}^{d\times k}\!\to\!\mathbb{R}\) that is
\(L\) bounded as:
$
   \bigl\|\nabla_ZF_x(Z)\bigr\|_F\;\le\;L
   \quad\forall x,\;Z.
$
Define the regularised objective
$
   \mathcal{L}(W)
   :=
   \mathbb{E}_{x}\!\bigl[F_x(WP(x))\bigr]
   +\frac{\lambda}{2}\,\|W\|_F^{2},
   \qquad \lambda>0,
   \qquad W\in\mathbb{R}^{d\times k}.
$
Let \(Q\in\mathbb{R}^{d\times d}\) be an orthogonal projector and set
\(W_Q:=WQ\).
Introduce the \emph{average spectral overlap}
$
     p_Q \;:=\;
     \mathbb{E}_{x}\bigl[\|P(x)Q\|_2\bigr]\in[0,1].
$
Run gradient descent
$
   W^{(t+1)}
   = W^{(t)}
   - \eta\,\nabla_W\mathcal{L}(W^{(t)}),
   \qquad 0<\eta\lambda<1.
$
Then for every \(t\ge0\)
$
   \;
     \|W_Q^{(t)}\|_F
     \;\le\;
     (1-\eta\lambda)^{t}\,\|W_Q^{(0)}\|_F
     +\frac{p_QL}{\lambda}\bigl[1-(1-\eta\lambda)^{t}\bigr]
   \;
$
and consequently
\[
   \limsup_{t\to\infty}\|W_Q^{(t)}\|_F
   \;\le\;\frac{p_QL}{\lambda}.
\]
\end{theorem}

\begin{proof}
For \(Z := WP(x)\) set \(G(x):=\nabla_ZF_x(Z)\).
Because \(\nabla_WF_x(WP(x)) = G(x)P(x)\), adding the
\(\ell_2\)-regulariser gives
\[
   \nabla_W\mathcal{L}(W)
   \;=\;
   \mathbb{E}_{x}[G(x)P(x)] + \lambda W .
\]

\smallskip\noindent

Right–multiplying the GD update by \(Q\),
\[
   W_Q^{(t+1)}
   = (1-\eta\lambda)\,W_Q^{(t)}
   - \eta\,\mathbb{E}_{x}[G(x)P(x)]\,Q .
\]

\smallskip\noindent
Using the bound and \(\|P(x)Q\|_2\le1\),
\[
   \|G(x)P(x)Q\|_F\;\le\;L\,\|P(x)Q\|_{2}.
\]
Taking expectation and the definition of \(p_Q\),
\[
   \bigl\|\mathbb{E}_{x}[G(x)P(x)]Q\bigr\|_F
   \;\le\; L\,p_Q .
\]

\smallskip\noindent
Applying the bound in the recursion,
\[
   \|W_Q^{(t+1)}\|_F
   \;\le\;
   (1-\eta\lambda)\|W_Q^{(t)}\|_F
   + \eta L p_Q .
\]

\smallskip\noindent
The inhomogeneous geometric series yields
\[
   \|W_Q^{(t)}\|_F
   \le
   (1-\eta\lambda)^{t}\|W_Q^{(0)}\|_F
   + \frac{Lp_Q}{\lambda}\bigl[1-(1-\eta\lambda)^{t}\bigr] .
\]

\smallskip\noindent
Since \(0<1-\eta\lambda<1\),
\(\displaystyle\lim_{t\to\infty}(1-\eta\lambda)^t=0\), giving
\(\displaystyle\limsup_{t\to\infty}\|W_Q^{(t)}\|_F\le p_QL/\lambda\).
\end{proof}

Finally, we extend the above result to the case where $U$ is data dependent and the networks is optimized via stochastic mini-batch gradient descent.

\begin{theorem}
\label{thm:subspace-decay-sgd}
For every sample \(x\) let  
$
   U(x)\in\mathbb{R}^{d\times r},
   \qquad
   U(x)^{\top}U(x)=I_r,
   \qquad
   P(x):=U(x)U(x)^{\top}\in\mathbb{R}^{d\times d}.
$
Fix a loss family \(F_x:\mathbb{R}^{d\times k}\!\to\!\mathbb{R}\) that is
\(L\) bounded as:
$
   \bigl\|\nabla_ZF_x(Z)\bigr\|_F\;\le\;L
   \quad\forall x,\;Z.
$
Define the regularised objective
$
   \mathcal{L}(W)
   :=
   \mathbb{E}_{x}\!\bigl[F_x(WP(x))\bigr]
   +\frac{\lambda}{2}\,\|W\|_F^{2},
   \qquad \lambda>0,
   \qquad W\in\mathbb{R}^{d\times k}.
$
Let \(Q\in\mathbb{R}^{d\times d}\) be an orthogonal projector and set
\(W_Q:=WQ\).
Introduce the \emph{average spectral overlap}
$
     p_Q \;:=\;
     \mathbb{E}_{x}\bigl[\|P(x)Q\|_2\bigr]\in[0,1].
$ 
At iteration $t=0,1,\dots$ draw an i.i.d.\ sample $x_t$ and perform
\[
   W^{(t+1)}
   \;=\;
   W^{(t)}
   -\eta_t\,g_t,
   \qquad
   g_t
   :=\nabla_{W}F_{x_t}\!\bigl(W^{(t)}P(x_t)\bigr)P(x_t)
     +\lambda\,W^{(t)},
\]
with stepsizes $\eta_t>0$ satisfying $\eta_t\lambda<1$.
Set $W_Q^{(t)}:=W^{(t)}Q$ and
$p_Q:=\mathbb{E}_{x}\bigl[\|P(x)Q\|_2\bigr]\in[0,1]$.

      If $\eta_t\equiv\eta$ and $0<\eta\lambda<1$, then for all $t\ge0$
      \[
         \mathbb{E}\!\bigl[\|W_Q^{(t)}\|_F\bigr]
         \;\le\;
         (1-\eta\lambda)^{t}\,\|W_Q^{(0)}\|_F
         +\frac{p_QL}{\lambda}\bigl[1-(1-\eta\lambda)^{t}\bigr]
      \]
      and hence
      \(
         \displaystyle
         \limsup_{t\to\infty}
         \mathbb{E}\!\bigl[\|W_Q^{(t)}\|_F\bigr]
         \le
         \dfrac{p_QL}{\lambda}.
      \)

 \end{theorem}

\begin{proof}
Let $\mathcal{F}_t:=\sigma\!\bigl(W^{(0)},\dots,W^{(t)}\bigr)$ be the
natural filtration.  All expectations $\mathbb{E}[\,\cdot\,]$ are taken over the
drawn samples $\{x_s\}_{s\le t}$.

Multiplying the update by $Q$ on the right gives
\[
      W_Q^{(t+1)}
      \;=\;
      W_Q^{(t)}-\eta_t\,g_tQ.
\]
Because $g_t$ is an \emph{unbiased} gradient estimate,
\(
      \mathbb{E}\!\bigl[g_t\mid\mathcal{F}_t\bigr]
      =\nabla_{W}\mathcal{L}\!\bigl(W^{(t)}\bigr).
\)

Take conditional expectation and use
\(
      \mathbb{E}\!\bigl[g_tQ\mid\mathcal{F}_t\bigr]
      =
      \mathbb{E}\!\bigl[\nabla_WF_{x_t}(\cdot)P(x_t)Q\mid\mathcal{F}_t\bigr]
      +\lambda W_Q^{(t)}:
\)
\[
\begin{aligned}
      \mathbb{E}\!\bigl[W_Q^{(t+1)}\mid\mathcal{F}_t\bigr]
      &= W_Q^{(t)}-\eta_t\lambda W_Q^{(t)}
           -\eta_t\,
             \mathbb{E}\!\bigl[\nabla_WF_{x_t}(\cdot)P(x_t)Q
                       \mid\mathcal{F}_t\bigr].
\end{aligned}
\]
Apply the Frobenius norm and the triangle inequality:
\[
      \mathbb{E}\!\bigl[\|W_Q^{(t+1)}\|_F\mid\mathcal{F}_t\bigr]
      \le (1-\eta_t\lambda)\|W_Q^{(t)}\|_F
           +\eta_t\,
             \Bigl\|
                \mathbb{E}\!\bigl[\nabla_WF_{x_t}(\cdot)P(x_t)Q
                                  \mid\mathcal{F}_t\bigr]
             \Bigr\|_F.
\]
Since $\|\nabla_ZF_x(Z)\|_F\le L$ for every $x$ and $Z$,
\[
      \bigl\|\nabla_WF_{x_t}(\cdot)P(x_t)Q\bigr\|_F
      \le
      L\,\|P(x_t)Q\|_2.
\]
Hence
\(
      \bigl\|
        \mathbb{E}\!\bigl[\nabla_WF_{x_t}(\cdot)P(x_t)Q
          \mid\mathcal{F}_t\bigr]
      \bigr\|_F
      \le Lp_Q .
\)
Thus
\begin{equation}\label{eq:cond-step-full}
      \mathbb{E}\!\bigl[\|W_Q^{(t+1)}\|_F\mid\mathcal{F}_t\bigr]
      \le
      (1-\eta_t\lambda)\|W_Q^{(t)}\|_F +\eta_t p_Q L .
\end{equation}

Let $\eta_t\equiv\eta$.
Taking full expectation of \eqref{eq:cond-step-full} yields
\(
      y_{t+1}\le (1-\eta\lambda)\,y_t + \eta p_Q L,
\)
where $y_t:=\mathbb{E}\!\bigl[\|W_Q^{(t)}\|_F\bigr]$.
Solving the linear recurrence gives
\[
      y_t
      \le
      (1-\eta\lambda)^{t}\,y_0
      +\frac{\eta p_Q L}{\lambda}
         \bigl[1-(1-\eta\lambda)^{t}\bigr].
\]
With $y_0=\|W_Q^{(0)}\|_F$ we obtain the stated bound; letting
$t\to\infty$ yields $\tfrac{p_QL}{\lambda}$.


\end{proof}

\section{Attention output analysis}
\label{app:rank}

We investigate the attention output activation rank structure of pretrained frontier open-weight models, specifically focusing on prominent architectures such as LLaMA \cite{touvron2023llama}, Qwen \cite{qwen2023}, and Olmo \cite{olmo2024} (Fig.~\ref{fig:llama}, \ref{fig:qwen}, and \ref{fig:olmo}, reepectively). Interestingly, our empirical analyses reveal a consistently observed low-rank structure across these diverse model families, suggesting that this phenomenon is intrinsic to transformer-based architectures rather than specific to certain model training procedures or datasets.

The low-rank behavior of attention outputs in transformers has garnered considerable interest, as it significantly impacts model efficiency, interpretability, and the potential for compression. Previous studies have documented similar findings; notably, Dong et al. \cite{dong2021attention} observed substantial rank reduction in the self-attention matrices of transformers during training, attributing it to implicit regularization effects induced by the training dynamics. Similarly, Abbe and colleagues \cite{abbe2024transformers} provided theoretical insights, demonstrating that rank collapse in attention mechanisms naturally arises from the iterative nature of gradient-based optimization processes.

Recent advancements in understanding transformer geometry and optimization further support our observations. Sanyal et al. \cite{sanyal2024inheritune} reported that transformers inherently favor lower-dimensional subspaces in their activations, leading to stable rank reduction, particularly in large-scale models. This intrinsic property has been leveraged in various model compression schemes, where exploiting the low-rank structure of attention outputs allows for significant reductions in memory footprint and computational overhead without adversely affecting model accuracy \cite{frantar2023gptq, zhao2024galore}.

Our findings confirm and extend these results by highlighting that the low-rank structure is not only prevalent across different model architectures but also robustly present across models trained on diverse datasets and with varying parameter scales. This observation underscores the potential universality of the low-rank phenomenon in transformer-based language models, further suggesting avenues for universally applicable model compression techniques and efficient parallelization strategies in decentralized training scenarios. 

\begin{figure}[h]
    \centering
    \begin{minipage}[b]{0.32\linewidth}
        \centering
        \includegraphics[width=\linewidth]{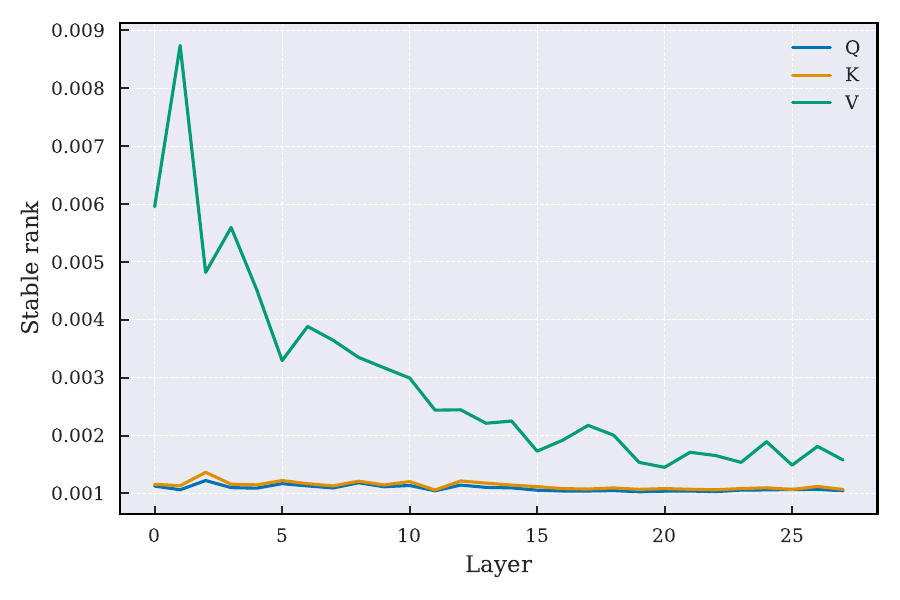}
        \caption*{\textsc{LlaMa-3.2-3B}}
    \end{minipage}
    \hfill
    \begin{minipage}[b]{0.32\linewidth}
        \centering
        \includegraphics[width=\linewidth]{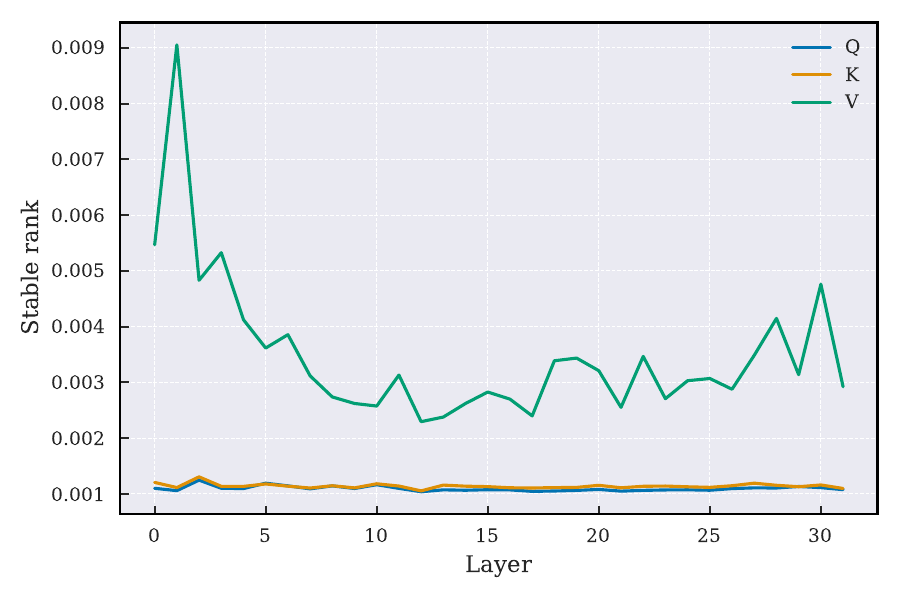}
        \caption*{\textsc{LlaMa-3.1-8B}}
    \end{minipage}
    \hfill
    \begin{minipage}[b]{0.32\linewidth}
        \centering
        \includegraphics[width=\linewidth]{figures/Meta-Llama-3-70B.pdf}
        \caption*{\textsc{LlaMa-3-70B}}
    \end{minipage}
    \vspace{-6pt}
    \caption{Stable rank distribution of the attention activations across \textsc{LlaMa 3} models normalized by their maximum possible rank. }
    \label{fig:llama}
\end{figure}

\begin{figure}[h]
    \centering
    \begin{minipage}[b]{0.32\linewidth}
        \centering
        \includegraphics[width=\linewidth]{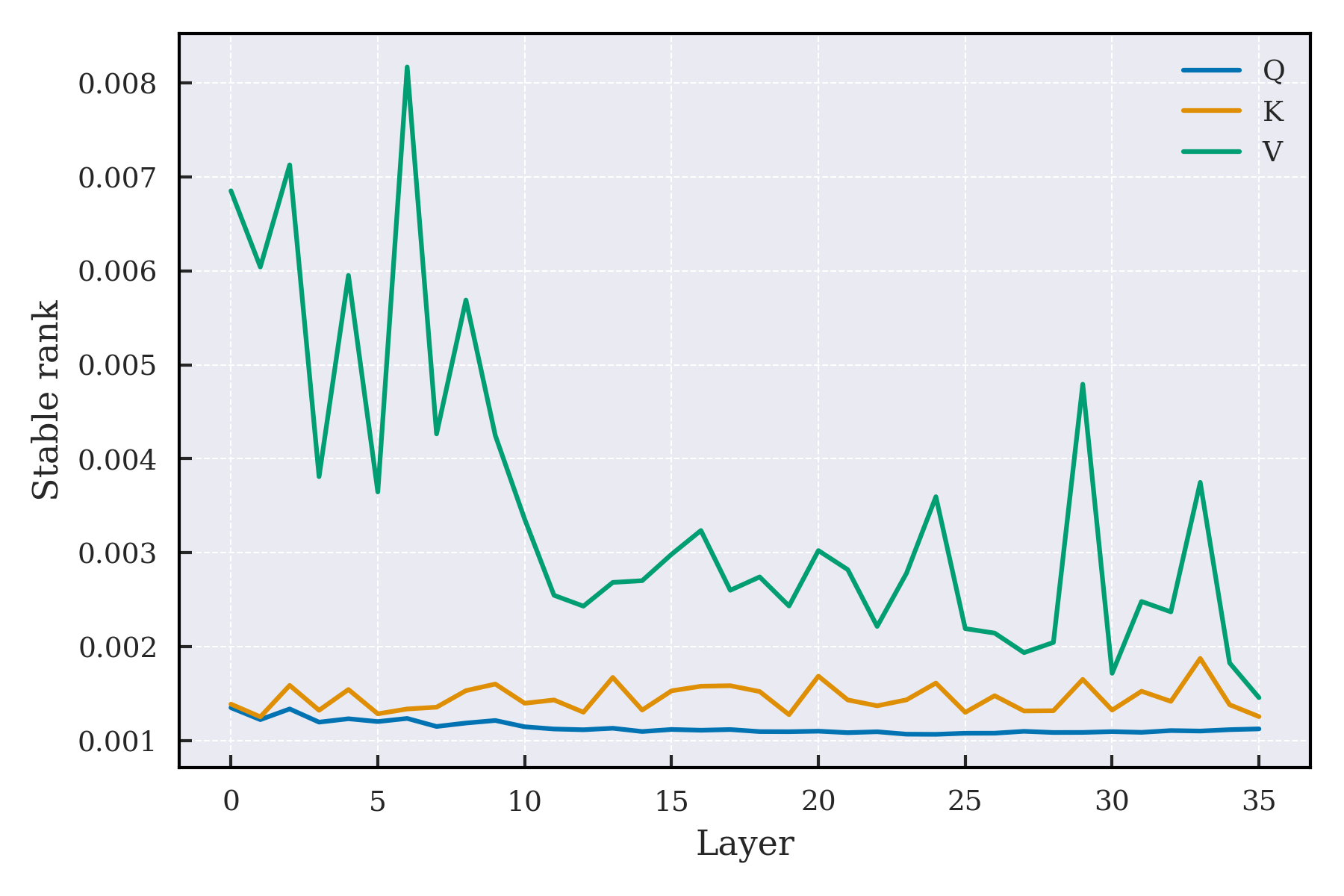}
        \caption*{\textsc{Qwen-3-8B}}
    \end{minipage}
    \hfill
    \begin{minipage}[b]{0.32\linewidth}
        \centering
        \includegraphics[width=\linewidth]{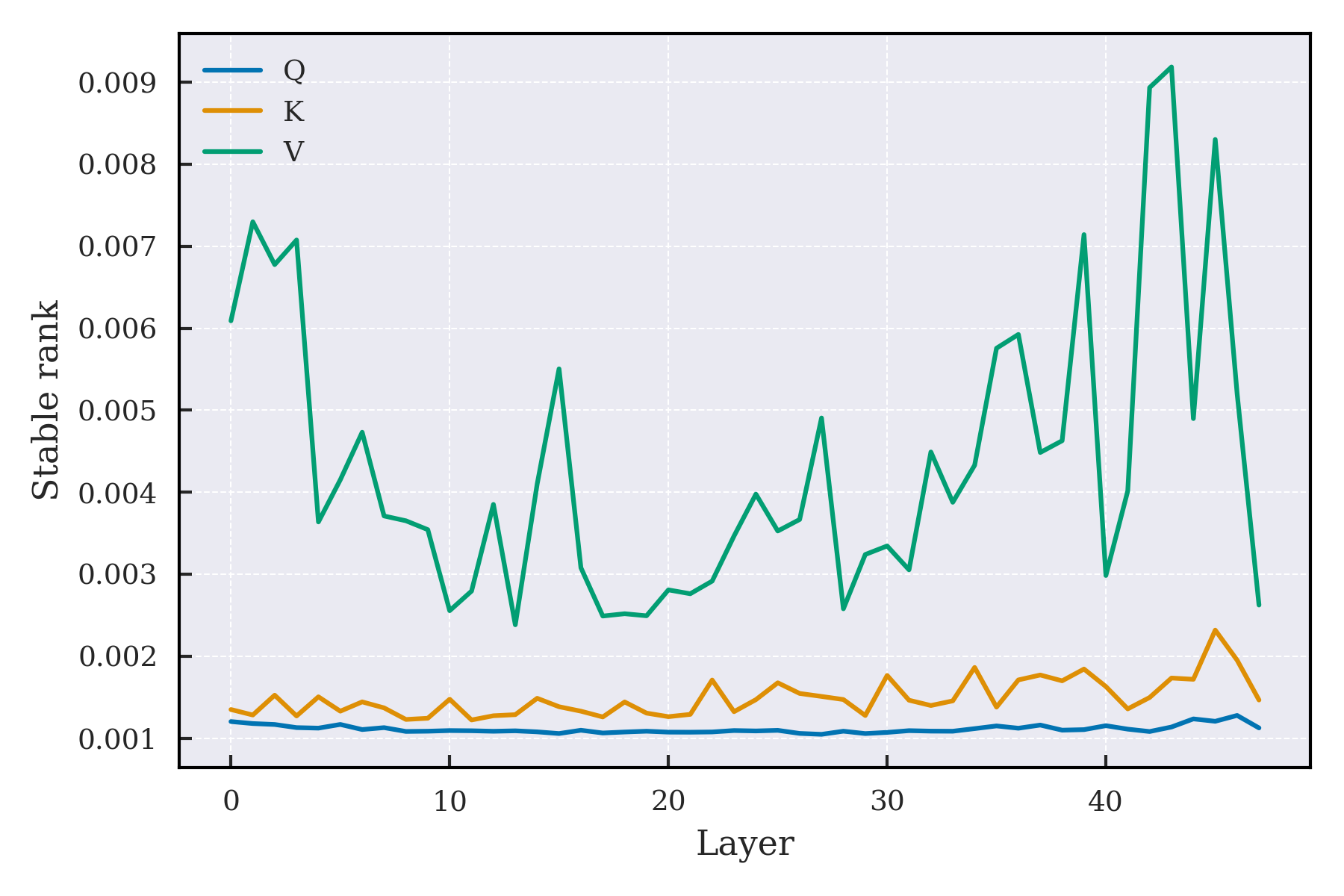}
        \caption*{\textsc{Qwen-3-30B}}
    \end{minipage}
    \hfill
    \begin{minipage}[b]{0.32\linewidth}
        \centering
        \includegraphics[width=\linewidth]{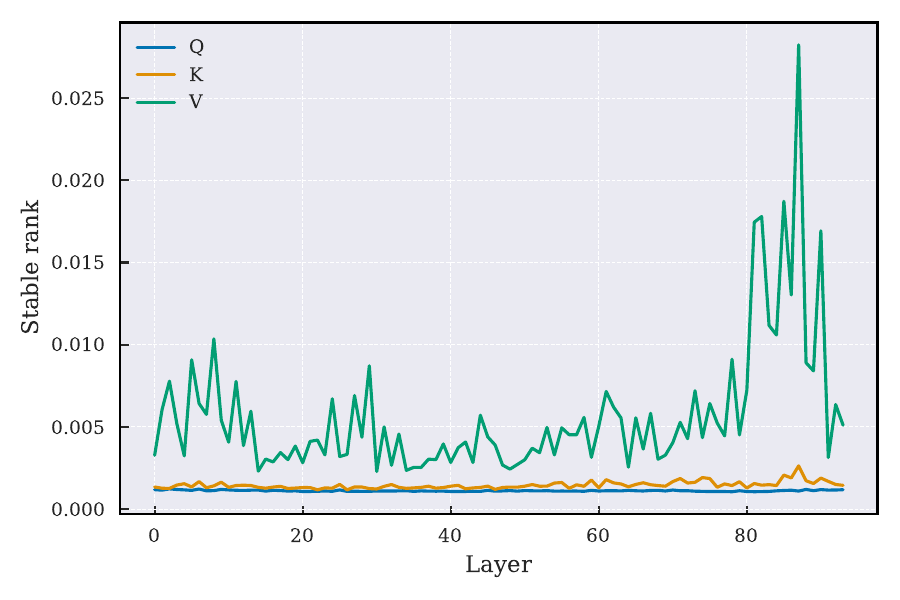}
        \caption*{\textsc{Qwen-3-235B}}
    \end{minipage}
    \vspace{-6pt}
    \caption{Stable rank distribution of the attention activations across \textsc{Qwen 3} models normalized by their maximum possible rank. }
    \label{fig:qwen}
\end{figure}

\begin{figure}[h]
    \centering
    \begin{minipage}[b]{0.32\linewidth}
        \centering
        \includegraphics[width=\linewidth]{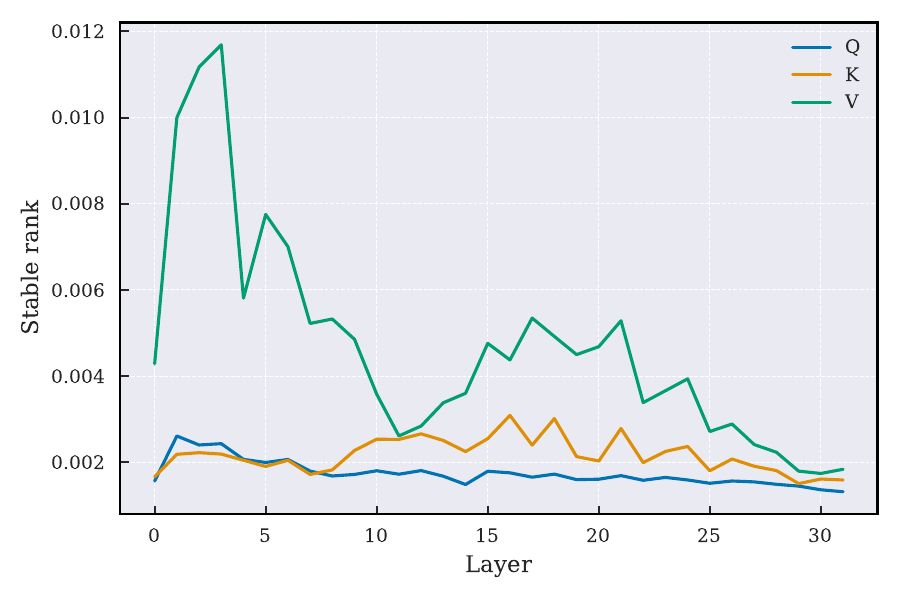}
        \caption*{\textsc{Olmo-2-7B}}
    \end{minipage}
    \hfill
    \begin{minipage}[b]{0.32\linewidth}
        \centering
        \includegraphics[width=\linewidth]{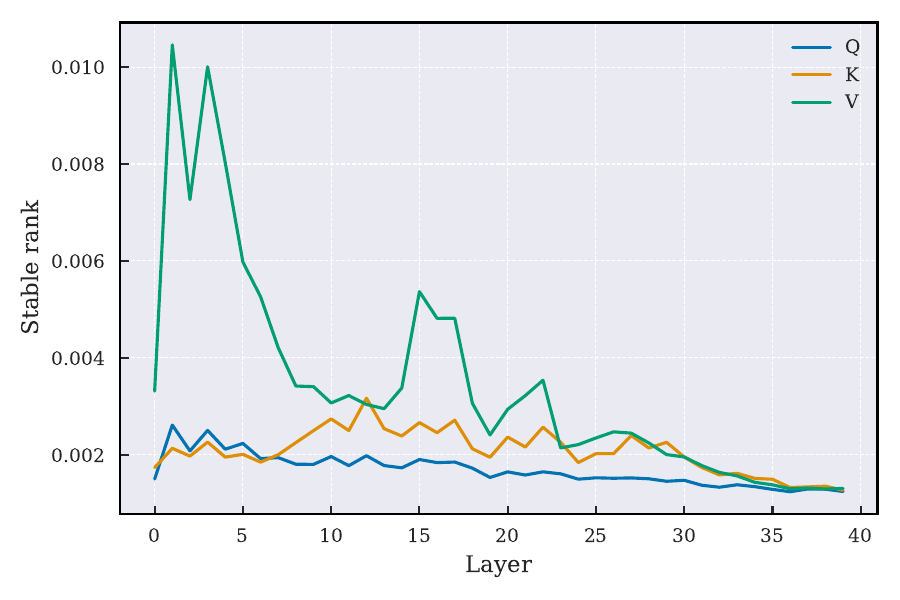}
        \caption*{\textsc{Olmo-2-13B}}
    \end{minipage}
    \hfill
    \begin{minipage}[b]{0.32\linewidth}
        \centering
        \includegraphics[width=\linewidth]{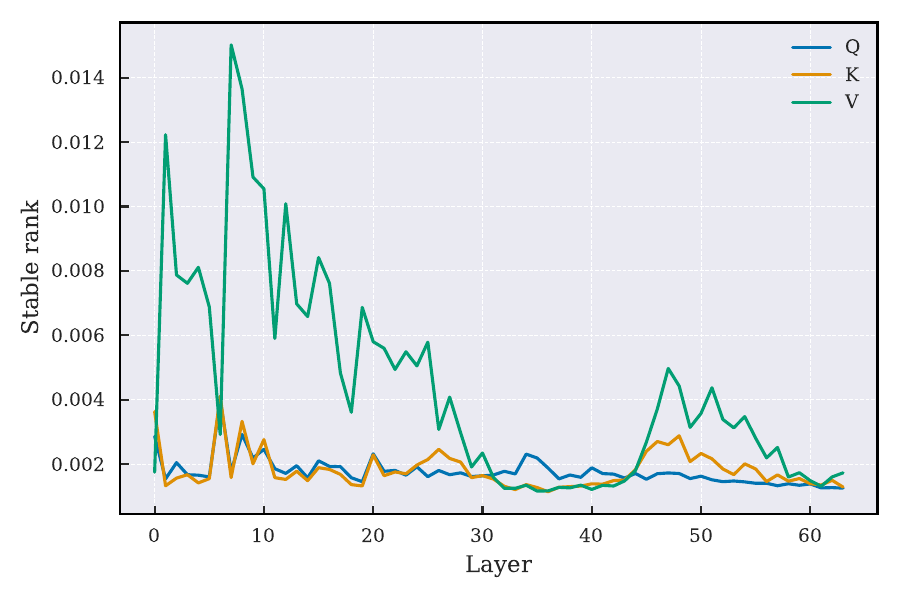}
        \caption*{\textsc{Olmo-2-32B}}
    \end{minipage}
    \vspace{-6pt}
    \caption{Stable rank distribution of the attention activations across \textsc{Olmo 2} models normalized by their maximum possible rank. }
    \label{fig:olmo}
\end{figure}

\section{Ablations}
\label{app:ablations}

We conduct ablation studies across different network architectures to validate that the effectiveness of our compression scheme is not dependent on specific design choices.

Figures~\ref{fig:8layers} and~\ref{fig:32layers} show the convergence behavior of 8-layer and 32-layer models, respectively, both using 8 attention heads. Similarly, Figures~\ref{fig:8heads} and~\ref{fig:24heads} compare convergence between models with 8 and 24 attention heads, respectively; both models have 8 layers. All four configurations use an embedding dimension of 2048.

Across all these variations, our compression method maintains convergence comparable to that of the centralized baseline, while achieving throughput competitive with centralized context-parallel models using 100Gbps connections—even though we only utilize 300Mbps connections.

In addition, we trained a model with an ultra-long context length of 256K tokens, distributed over 32 A100 GPUs. This model has 8 layers and 8 attention heads. As shown in Figure~\ref{fig:256h}, even in this extreme setting, our compression technique enables decentralized training to match the convergence of the centralized baseline.

Note that we plot the loss curves against iterations (not wall-clock time) and hence, 300mbps decentralized (uncompressed) and 100Gbps centralized curves overlap. Throughput is reported next to the curves separately.

\begin{figure}[h]
  \centering
  \begin{minipage}[t]{0.48\textwidth}
    \centering
    \includegraphics[width=\linewidth]{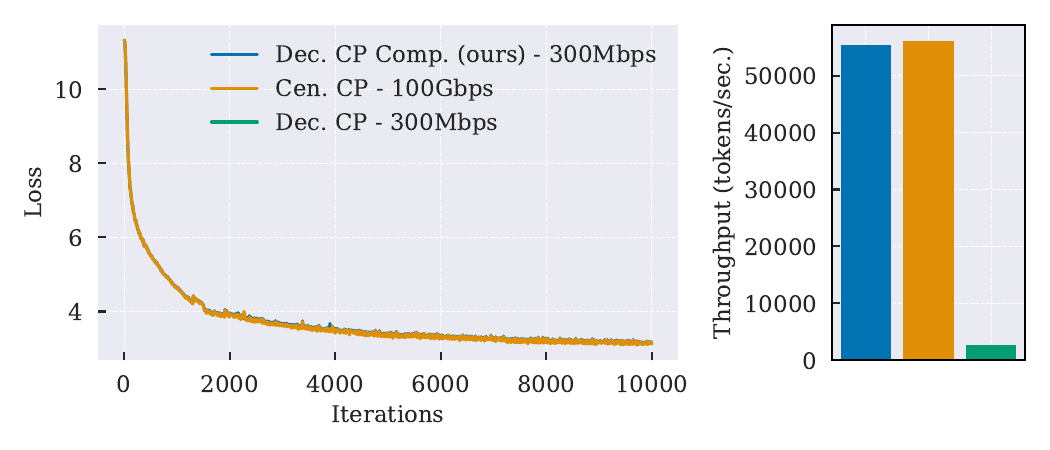}
    \caption{Loss over training iterations for 8-layer models (800M).}
    \label{fig:8layers}
  \end{minipage}
  \hfill
  \begin{minipage}[t]{0.48\textwidth}
    \centering
    \includegraphics[width=\linewidth]{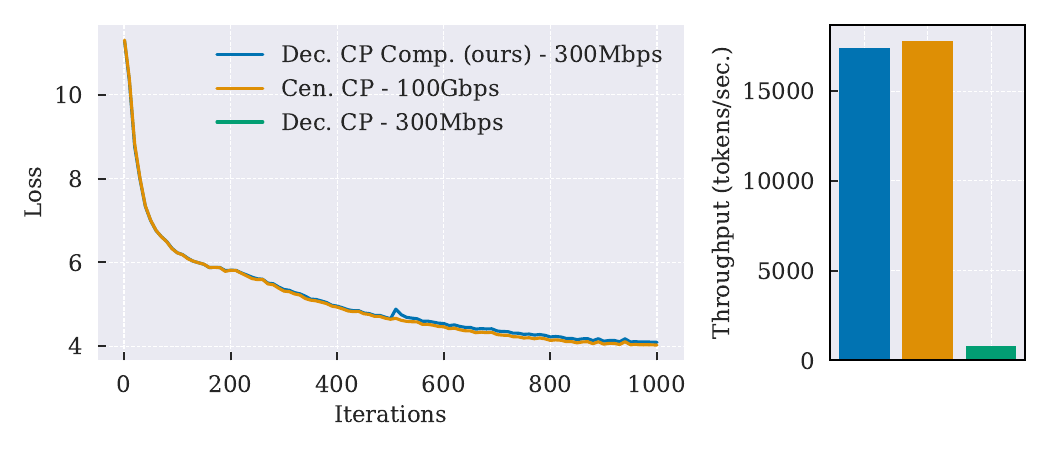}
    \caption{Loss over training iterations for 32-layer models (3B).}
    \label{fig:32layers}
  \end{minipage}
\end{figure}

\begin{figure}[h]
  \centering
  \begin{minipage}[t]{0.48\textwidth}
    \centering
    \includegraphics[width=\linewidth]{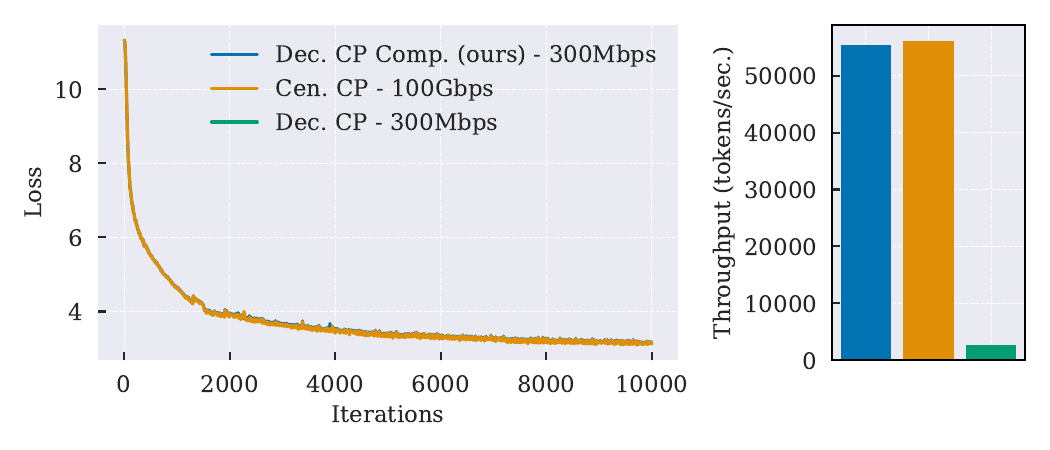}
    \caption{Loss over training iterations for 8-head models (800M).}
    \label{fig:8heads}
  \end{minipage}
  \hfill
  \begin{minipage}[t]{0.48\textwidth}
    \centering
    \includegraphics[width=\linewidth]{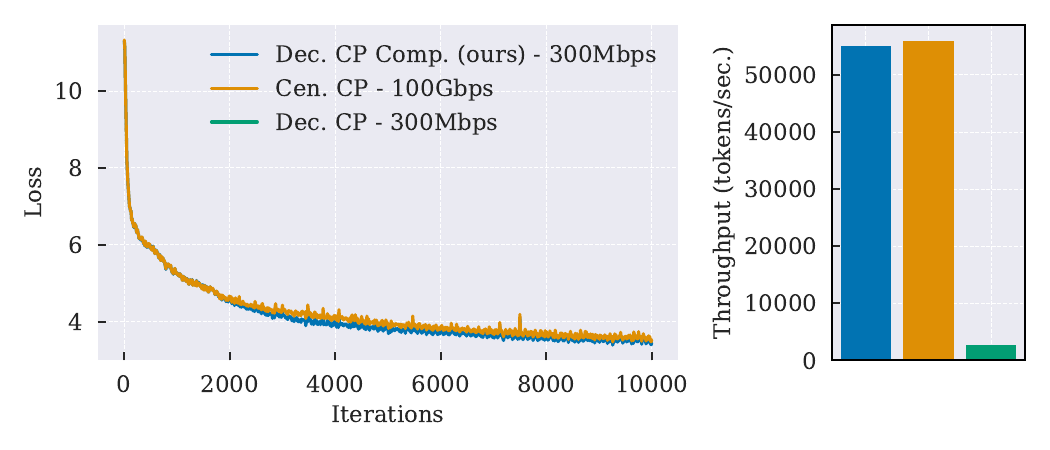}
    \caption{Loss over training iterations for 24-head models (2.5B).}
    \label{fig:24heads}
  \end{minipage}
\end{figure}

\begin{figure}[h]
    \centering
    \includegraphics[width=0.7\linewidth]{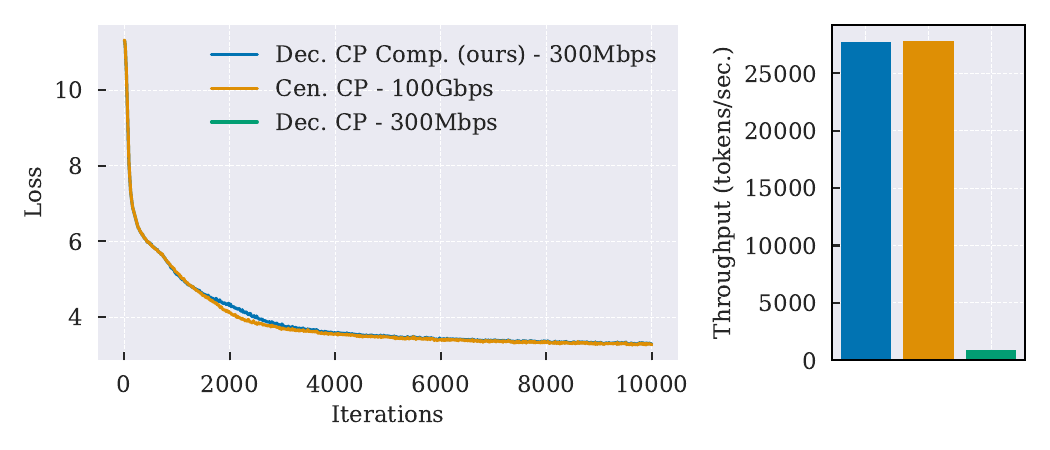}
    \caption{\textbf{Loss over training iterations for 256K context length training.} Our compression preserves the convergence with longer context lengths.}
    \label{fig:256h}
\end{figure}

\begin{figure}[h]
    \centering
    \includegraphics[width=0.7\linewidth]{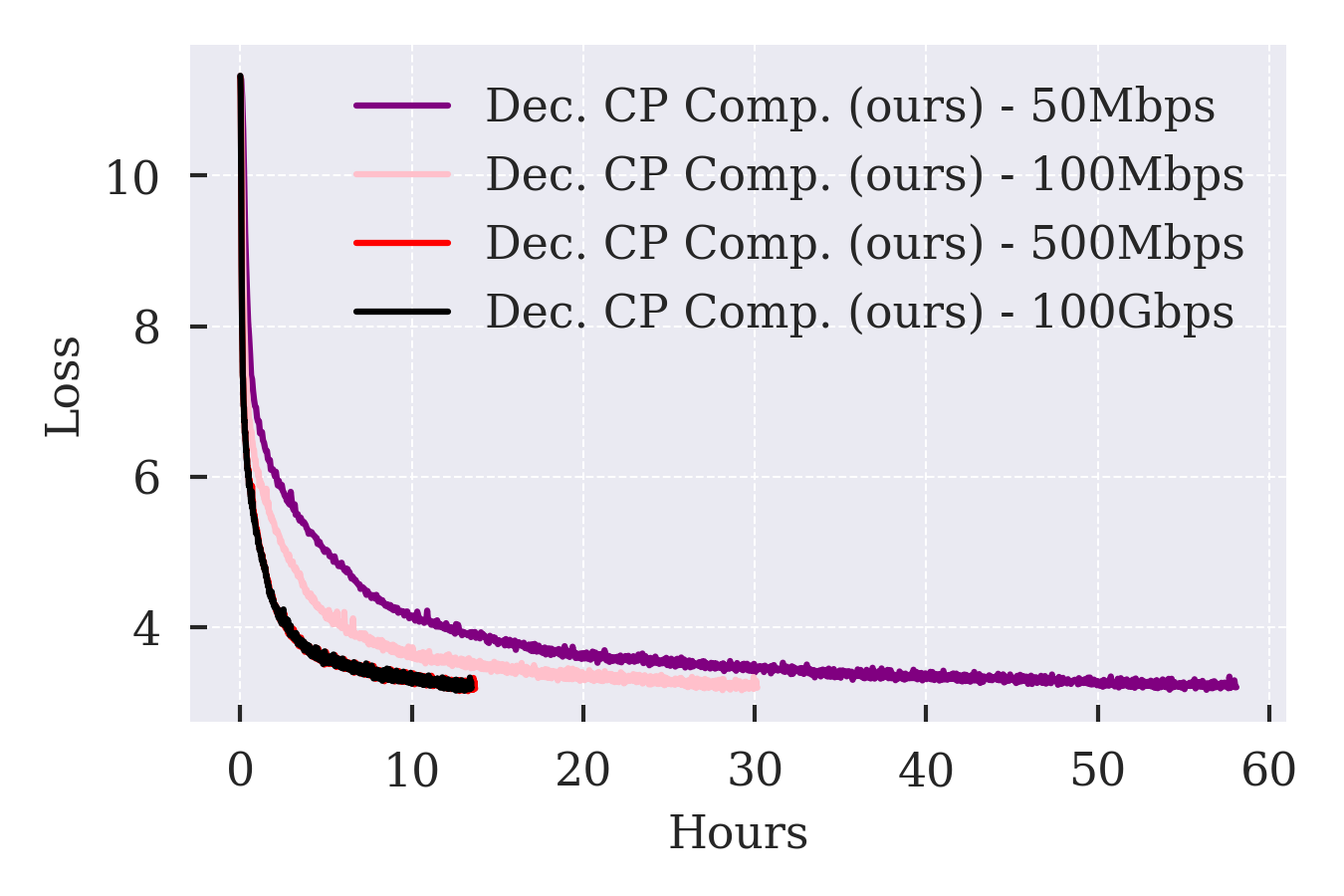}
    \caption{\textbf{Loss over wall-clock time across different bandwidths.} With our compression, the decentralized models achieve the convergence rate of centralized models at lower bandwidths. }
    \label{fig:bandwidth}
\end{figure}

\begin{table}[h!]
\centering
\caption{Effect of Unplugging Step on Perplexity (PPL)}
\begin{tabular}{l c}
\toprule
\textbf{Unplugging Step} & \textbf{PPL} \\
\midrule
Not unplugged & 22.64 \\
5k             & 22.69 \\
7k             & 22.64 \\
8k             & 22.71 \\
9k             & 22.71 \\
9.5k           & 22.77 \\
9.9k           & 23.12 \\
\bottomrule
\end{tabular}
\label{tab:unplugging}
\end{table}

\subsection{Removing the projection components}

In our experiments, we found that removing the projection components around the last 30\%-5\% range of training consistently results in stable performance across diverse settings. Based on this observation, we recommend using the last 30\%-5\% range of training as a simple and effective heuristic. Further, even when the projection components are removed earlier in training, the model is able to recover quickly, indicating a degree of robustness to the exact transition point. Nonetheless, the last 30\% heuristic offers a practical, safer, and an automatable guideline. We present an ablation on FineWeb in Table \ref{tab:unplugging} to further illustrate this (total steps 10k). Variations correspond to the third decimal point in the validation loss in most cases.

\appendix

\end{document}